%% file: main.tex
\documentclass[conference]{IEEEtran}
\IEEEoverridecommandlockouts
\usepackage{cite}
\usepackage{amsmath,amssymb,amsfonts}
\usepackage{algorithmic}
\usepackage{graphicx}
\usepackage{textcomp}
\usepackage{xcolor}

\usepackage{amssymb}
\usepackage{xcolor}
\usepackage{multirow}
\usepackage{amsthm}
\usepackage[linesnumbered,ruled]{algorithm2e}
\usepackage{listings}
\usepackage{subfigure}
\usepackage{nicematrix}
\usepackage{bbm}
\usepackage{url}
\usepackage{fontawesome}

\newtheorem{theorem}{Theorem}%
\newtheorem{definition}{Definition}%

\newcommand{\figurewidthone}{40mm}
\newcommand{\figurewidthtwo}{40mm}
\newcommand{\figurehspace}{3mm}

\def\BibTeX{{\rm B\kern-.05em{\sc i\kern-.025em b}\kern-.08em
    T\kern-.1667em\lower.7ex\hbox{E}\kern-.125emX}}

\renewcommand{\baselinestretch}{0.94}

\makeatletter
\newcommand{\linebreakand}{%
  \end{@IEEEauthorhalign}
  \hfill\mbox{}\par
  \mbox{}\hfill\begin{@IEEEauthorhalign}
}
\makeatother

\begin{document}

\title{Training Fair Models in Federated Learning without Data Privacy Infringement}

\author{
\IEEEauthorblockN{Xin Che}
\IEEEauthorblockA{McMaster University\\Hamilton, Canada\\
chex5@mcmaster.ca}
\and
\IEEEauthorblockN{Jingdi Hu}
\IEEEauthorblockA{East China University of Science\\ and Technology, Shanghai, China\\
22013928@mail.ecust.edu.cn}
\and
\IEEEauthorblockN{Zirui Zhou}
\IEEEauthorblockA{Huawei Technologies Canada\\
Burnaby, Canada\\
zirui.zhou@huawei.com}
\linebreakand
\IEEEauthorblockN{Yong Zhang}
\IEEEauthorblockA{Huawei Technologies Canada\\
Burnaby, Canada\\
yong.zhang3@huawei.com}
\and
\IEEEauthorblockN{Lingyang Chu*\thanks{*Corresponding author: Lingyang Chu.}}
\IEEEauthorblockA{McMaster University\\
Hamilton, Canada\\
chul9@mcmaster.ca}
}

\maketitle

\begin{abstract}
Training fair machine learning models becomes more and more important.  As many powerful models are trained by collaboration among multiple parties, each holding some sensitive data, it is natural to explore the feasibility of training fair models in federated learning so that the fairness of trained models, the data privacy of clients, and the collaboration between clients can be fully respected simultaneously.  However, the task of training fair models in federated learning is challenging, since it is far from trivial to estimate the fairness of a model without knowing the private data of the participating parties, which is often constrained by privacy requirements in federated learning. 
In this paper, we first propose a federated estimation method to accurately estimate the fairness of a model without infringing the data privacy of any party.  Then, we use the fairness estimation to formulate a novel problem of training fair models in federated learning. We develop FedFair, a well-designed federated learning framework, which can successfully train a fair model with high performance without data privacy infringement.  Our extensive experiments on three real-world data sets demonstrate the excellent fair model training performance of our method.
\end{abstract}


% === +++ paper content start
\input{Chapters/intro}

\input{Chapters/related_work}
\input{Chapters/problem}

\input{Chapters/solution}

\input{Chapters/exp}
\input{Chapters/con}

% === +++ paper content end

% ---- Bibliography ----
\bibliographystyle{IEEETran}
\bibliography{references_full}

\end{document}

%% file: Chapters/intro.tex
% !TEX root = ../main.tex
\section{Introduction}\label{sec1}

Fairness and collaboration are among the top priorities in machine learning and AI applications. As accurate machine learning models are deployed in more and more important applications, enhancing and ensuring fairness in such models becomes critical for AI for social good~\cite{pessach2020algorithmic}.

At the same time, in order to build accurate machine learning models for sophisticated applications, such as finance and medicare, many parties,
such as financial institutions and medical care organizations,
have to collaborate and contribute their own private data~\cite{kairouz2021advances,yang2019federated}. 
A key to successful collaborations is to fully protect the privacy of every party. 
However, it is a grand challenge to build fair machine learning models while preserving the data privacy of all parties in their collaboration.

There are many recent breakthroughs in enhancing fairness in machine learning models~\cite{pessach2020algorithmic}.  
As discussed in Section~\ref{sec:rw}, 
the existing fairness enhancing methods (please see~\cite{pessach2020algorithmic} for a survey) are mainly designed with an assumption of a unified available training data set, and thus cannot easily address the need for federated learning, such as distributed collaboration and privacy-preservation.
The classic federated learning methods~\cite{mcmahan2017communication,chen2017distributed,wang2020federated,huang2021personalized,yang2019federated} do not consider the fairness of machine learning models at all. 
The collaborative fairness methods~\cite{yang2017designing,gollapudi2017profit,yu2020fairness,mohri2019agnostic,li2019fair,lyu2020towards,sim2020collaborative} focus on the fairness of the collaboration relationship among participating parties in the federated learning process, which is substantially different from the goal of training fair models.

Recently, there are some initiative fairness-aware federated learning methods~\cite{du2020fairness,zhang2020fairfl,papadaki2022minimax,hu2022provably,zeng2021improving,mohri2019agnostic} conducted simultaneously with our work. 
%As discussed in Sections~\ref{sec:fair-federated} and~\ref{sec:pre}, those
These studies focus on improving the fairness of a trained model with respect to various notions of model fairness measurements that are substantially different from ours.

In this paper, we tackle the challenging problem of training fair models in federated learning.  We systematically model the problem and develop a comprehensive solution that theoretically guarantees both fairness and privacy preservation requirements.  Our key idea is to carefully design the fairness constraint based on a general fairness measurement named difference of generalized equal opportunity (DGEO)~\cite{donini2018empirical}, so that the fairness constraint can be enforced in the whole learning process in a federated manner. 
In other words, the fairness constraint is implemented in a fully collaborative and privacy preserving way.  Moreover, the problem of training fair models in federated learning presents a challenging optimization problem to maximize accuracy and ensure fairness.  To solve the problem we develop an elegant solution based on alternating gradient projection~\cite{xu2020unified} in optimization.  

We make the following contributions.
First, we propose a federated estimation method to accurately estimate the fairness of a model without infringing the data privacy of any party.
Second, when the data sets of all participating parties are identically and independently distributed (IID), we prove that the federated estimation method is more accurate than locally estimating the fairness of a model on individual parties and thus leads to a superior fair model training performance in federated learning.
Third, by incorporating the model fairness estimation as a fairness constraint with a widely adopted federated loss function~\cite{mcmahan2017communication}, we formulate our novel FedFair problem. We also develop an effective federated learning framework to tackle this problem without infringing the data privacy of any party.
Last, we conduct extensive experiments on three real-world data sets to demonstrate that our method achieves the best performance in both cases when the data sets of all participating parties are IID and are not IID.

%The rest of the paper is organized as follows.  We review the related works in Section~\ref{sec:rw}. In Section~\ref{sec:problem}, we formulate the problem of federated fair model training. We develop practical solutions in Section~\ref{sec:solution} and report systematic experimental results in Section~\ref{sec:exp}. We conclude the paper in Section~\ref{sec:con}.

%% file: Chapters/related_work.tex
% !TEX root = ../main.tex
\section{Related Works}\label{sec:rw}

There is a rich body of literature on federated learning. Please see~\cite{kairouz2021advances} and~\cite{yang2019federated} for excellent surveys.  Here, we focus on connecting fair models and federated learning.

The classical federated learning methods~\cite{mcmahan2017communication,chen2017distributed,wang2020federated,huang2021personalized,yang2019federated} are mostly designed to protect privacy instead of training fair models. A straightforward extension to training fair models is to apply fairness constraints locally on each participating party. However, as discussed later in Section~\ref{sec:tfem}, since the local fairness on each party does not provide any guarantee of global fairness on all parties, this leads to an inferior fairness performance due to the inaccurate model fairness measures computed locally on each participating party.

The collaborative fairness methods~\cite{yang2017designing,gollapudi2017profit,yu2020fairness,mohri2019agnostic,li2019fair,lyu2020towards,sim2020collaborative} focus on balancing the rewards paid to the participating parties based on their contributions to the federated learning process.
A typical line of works~\cite{yang2017designing,gollapudi2017profit,yu2020fairness} combine incentive schemes with game theory to determine the payoff received by each party commensurate according to their contributions to the training process.
Using a similar idea, Lyu~\textit{et~al.}~\cite{lyu2020towards} and Sim~\textit{et~al.}~\cite{sim2020collaborative} achieve a better collaborative fairness performance by directly using model accuracy as the reward for the parties.
Some other works~\cite{mohri2019agnostic,li2019fair} advocate egalitarian equity in collaboration by optimizing the performance of the party with the worst performance. 

The collaborative fairness task is substantially different from our study, because our task focuses on enhancing the fairness of a machine learning model, but the collaborative fairness task focuses on the fairness of contribution valuation among participating parties in the federated learning process.

A few recent works~\cite{du2020fairness,zhang2020fairfl,papadaki2022minimax,hu2022provably,zeng2021improving,mohri2019agnostic} simultaneous to our work attempt to develop federated learning frameworks to train fair models with respect to different notions of model fairness.
AgnosticFair~\cite{du2020fairness} and FedFB~\cite{zeng2021improving} attempt to train models that are fair with respect to demographic parity~\cite{calders2009building,dwork2012fairness,calders2010three}.
FairFL~\cite{zhang2020fairfl} improves model fairness by reducing the difference of F1-scores between two groups of data instances.
FLBGL~\cite{hu2022provably} employs a fairness constraint named bounded group loss, which focuses on upper bounding the loss of each group rather than the loss difference between any two groups.
FedMinMax~\cite{papadaki2022minimax} uses MinMax fairness to minimize the expected risk of the worst performing demographic group.

Unlike the above approaches, our work focuses on training fair models based on a different fairness measure named difference of generalized equal opportunities (DGEO)~\cite{donini2018empirical}.

%% file: Chapters/problem.tex
\section{Problem Formulation}\label{sec:problem}

In this section, we first review the fairness constraint based on a general fairness measure named \textbf{difference of generalized equal opportunities} (DGEO)~\cite{donini2018empirical}. 
Then, we formulate the problem of federated fair model training.

\subsection{The DGEO Constraint}\label{sec:DGEO}

Denote by $(\mathbf{x}, s, y)$ a triple of random variables drawn from an unknown distribution $\mathcal{D}$, where $\mathbf{x}\in\mathbb{R}^d$ is a $d$-dimensional vector of features, $s\in\{a, b\}$ is the membership of $\mathbf{x}$ between a pair of pre-defined \textbf{protected groups} (e.g., $a$ for ``female'' and $b$ for ``male''), and $y\in\{1, \ldots, C\}$ is the class label of $\mathbf{x}$. 
A sample of $(\mathbf{x}, s, y)$ is called a \textbf{data instance} (\emph{instance} for short).

Denote by $f_\theta:\mathbb{R}^d\rightarrow\mathbb{R}$ a model parameterized by a set of parameters $\theta$. 
For a loss function $\ell(f_\theta(\mathbf{x}), y)\in\mathbb{R}$,  the DGEO~\cite{donini2018empirical} of $f_\theta$ with respect to the instances contained in the protected groups $\{a, b\}$ and belonging to a \textbf{protected class} $c\in\{1, \ldots, C\}$ is defined as follows.

\begin{definition}%[DGEO~\cite{donini2018empirical}]
Let $L^{s,c}(\theta)=\mathbb{E}[{\ell(f_\theta(\mathbf{x}), y) | s, y = c}]$ be the expected conditional loss of a model $f_\theta$ on the samples in a protected group $s$ with a protected class label $c$. The \textbf{difference of generalized equal opportunity (DGEO)} of the model $f_\theta$ with respect to the class $c$ is the absolute difference between $L^{a,c}(\theta)$ and $L^{b,c}(\theta)$, that is, $\left|{L^{a,c}(\theta) - L^{b,c}(\theta)}\right|$.
\end{definition}

The above DGEO is a continuous function to measure the fairness of two protected groups.
A smaller DGEO means that $f_\theta$ commits more similar errors on the two protected groups of instances contained in a class $c$, which further indicates $f_\theta$ is fairer with respect to the instances in class $c$~\cite{donini2018empirical}.
On the opposite, a larger DGEO means that $f_\theta$ is more unfair.

Donini~\textit{et~al.}~\cite{donini2018empirical} propose the notion of \textbf{DGEO constraint}. For a given threshold $\epsilon \geq 0$ indicating the maximum unfairness we can tolerate, a model $f_\theta$ is said to be \textbf{$\epsilon$-fair} if it satisfies
\begin{equation}\label{eq:dgeo_constraint}
\left|{L^{a,c}(f_\theta) - L^{b,c}(f_\theta)}\right| \leq \epsilon.
\end{equation}

\subsection{Federated Fair Model Training}

Now we formulate the task of training fair models in federated learning (\emph{federated fair model training} for short).
The task allows a group of parties called \textbf{clients} to collaborate through a server to train a fair and accurate model without infringing the data privacy of any client.

Denote by $U_1, \ldots, U_N$ a group of $N$ clients, 
by $B_1, \ldots, B_N$ the private data sets owned by the clients, respectively, 
by $m_i$ the number of data instances in $B_i$, and
by $\mathcal{D}$ the unknown data distribution of $\cup_{i=1}^N B_i$, the union of the data of all clients.
We define the federated fair model training task as follows.

\begin{definition}[Federated Fair Model Training Task]
Given a pair of protected groups $\{a, b\}$, a protected class $c\in\{1, \ldots, C\}$, and a threshold $\epsilon \geq 0$ indicating the maximum unfairness that can be tolerated,
the \textbf{federated fair model training task} is to develop a method that enables a group of clients $U_1, \ldots, U_N$ to train a model $f_\theta$ in collaboration, such that
\begin{enumerate}
	\item The data instances in the private data set of each client, as well as any information about the distribution of the private data set, are not exposed to any other clients or any third party; and
	\item $f_\theta$ is $\epsilon$-fair for the instances in the class $c$.
	
\end{enumerate}
\end{definition}

The first requirement is exactly the data-privacy-preserving requirement in conventional federated learning~\cite{mcmahan2017communication,chen2017distributed,wang2020federated,huang2021personalized,yang2019federated}.
Here, we consider an honest but curious setting~\cite{kadhe2020fastsecagg} for the data privacy of clients.
Being honest means that every client and the server strictly implements the proposed federated learning algorithm without performing any malicious operation to disrupt the learning process.
Being curious means that the clients and the server are allowed to actively collect information from all the contents (i.e., model parameters, gradients, etc.) they legally receive during the federated learning process.

The second requirement is about the fairness of the collaboratively trained model with respect to input data instances.  
Please note that, in this paper, we are not concerned about the fair valuation of contributions from clients, which is a completely different topic orthogonal to ours.

The above federated fair model training task can be formulated as a constrained optimization problem as follows.
\begin{subequations}\label{eq:04}
\begin{align}
	&\min_{\theta} L(\theta) \label{eq:04a} \\
	&\mbox{s.t. } \left|{L^{a, c}(\theta)  -  L^{b, c}(\theta)}\right| \le \epsilon \label{eq:04b}, 
\end{align}
\end{subequations}
where $L(\theta)=\mathbb{E}[{\ell(f_\theta(\mathbf{x}), y)}]$ is the expected loss of $f_\theta$ over the unknown distribution $\mathcal{D}$, and the DGEO constraint in Equation~\eqref{eq:04b} requires the trained model $f_\theta$ to be $\epsilon$-fair.

It is difficult to find a solution to Equation~\eqref{eq:04}, because we cannot compute the exact values of $L(\theta)$ and $L^{a, c}(\theta)  -  L^{b, c}(\theta)$ without knowing $\mathcal{D}$. To find a practical solution, we need to estimate the values of $L(\theta)$ and $L^{a, c}(\theta)  -  L^{b, c}(\theta)$ while preserving the data privacy of all the clients.

%% file: Chapters/solution.tex
% !TEX root = ../main.tex
\section{FedFair: A Practical Solution}\label{sec:solution}
In this section, we develop a practical solution to the federated fair model training task. 
We start with an analysis of the problem. 
Then, we present a baseline method that estimates $L^{a, c}(\theta) - L^{b, c}(\theta)$ locally on each client.
Last, we develop our practical solution that estimates $L^{a, c}(\theta) - L^{b, c}(\theta)$ in a federated manner and discuss its advantage over the baseline method.

\subsection{Problem Analysis}
To develop a practical solution to the federated fair model training problem, let us first analyze the opportunities and challenges in estimating $L(\theta)$ and $L^{a, c}(\theta) - L^{b, c}(\theta)$.

Similar to many classical federated learning methods~\cite{mcmahan2017communication,chen2017distributed,wang2020federated,huang2021personalized}, we can compute a \textbf{federated loss} that estimates $L(\theta)$ by 
\begin{equation}
\label{eq:lf}
\sum_{i=1}^N \frac{m_i}{m_{\mbox{total}}} \hat L_i(\theta),
\end{equation}
where $m_{\mbox{total}} = \sum_{i=1}^N m_i$ is the total number of data instances owned by all clients, and
\begin{equation}
\hat L_i(\theta) = \frac{1}{m_i}\sum_{(\mathbf{x}, s, y)\in B_i}\ell(f_\theta(\mathbf{x}), y)
\end{equation}
is the empirical loss of $f_\theta$ on the private data set $B_i$.

Computing the federated loss does not infringe the data privacy of any client, because each empirical loss $\hat L_i(\theta)$ is privately computed by the corresponding client $U_i$, and all empirical losses are sent to the server to compute the federated loss.
For every client, the server only knows the empirical loss of the client and the number of instances in the private data set. In other words, the client does not expose to the server any instance or any data distribution information.

The real challenge is to estimate $L^{a, c}(\theta) - L^{b, c}(\theta)$ without infringing the data privacy of any client. 
Without a carefully designed method to estimate $L^{a, c}(\theta) - L^{b, c}(\theta)$, a client may expose to the server some distribution information about the sensitive attribute value or the target class.

In the rest of this section, we first introduce a baseline method to estimate $L^{a, c}(\theta) - L^{b, c}(\theta)$ and discuss the limitations of this method.
Then, we present our federated estimation method to estimate $L^{a, c}(\theta) - L^{b, c}(\theta)$ more accurately.
Last, we develop an optimization method to find a good solution to the federated fair model training task.

\subsection{A Local Estimation Approach}
To estimate $L^{a, c}(\theta)  -  L^{b, c}(\theta)$ while keeping the data privacy of all clients, a baseline method is to let every client $U_i$ use its private data set $B_i$ to compute its \textbf{local estimation} by
%\begin{equation}
$\hat L_i^{a, c}(\theta) - \hat L_i^{b, c}(\theta)$,
%\end{equation}
where $\hat L_i^{a, c}(\theta)$ and $\hat L_i^{b, c}(\theta)$ are the estimators of the expected conditional losses $L^{a,c}(\theta)$ and $L^{b,c}(\theta)$, respectively.

The estimators $\hat L_i^{a, c}(\theta)$ and $\hat L_i^{b, c}(\theta)$ 
are computed on the private data set $B_i$ by
\begin{equation}
	\hat L_i^{s,c}(\theta) = \frac{1}{m_i^{s, c}} \sum_{(\mathbf{x}, s, c)\in B_i} \ell(f_\theta(\mathbf{x}), c),
\end{equation}
where $s\in\{a, b\}$ indicates the group membership of a data instance $(\mathbf{x}, s, c)$, and $m_i^{s,c}$ is the number of data instances in $B_i$ that belong to a group $s$ and have a class label $c$.

We use the local estimations computed on the private data sets of all the clients to build $N$ \textbf{local DGEO constraints} as $\left|{\hat L_i^{a, c}(\theta) - \hat L_i^{b, c}(\theta)} \right| \leq \epsilon, \forall i\in\{1, \ldots, N\}$.

By incorporating these constraints into the federated loss in Equation~\eqref{eq:lf},
we convert the problem in Equation~\eqref{eq:04} into the following \textbf{locally constrained optimization (LCO) problem}.
\begin{subequations}\label{eq:oragin}
\begin{align}
	&\min_{\theta} \sum_{i=1}^N \frac{m_i}{m_{\mbox{total}}} \hat L_i(\theta) \label{simu:oragina}\\
	&\mbox{s.t. } \left|{\hat L_i^{a, c}(\theta) - \hat L_i^{b, c}(\theta)}\right| \le \epsilon, \forall i\in\{1, \ldots, N\} \label{simu:oraginab}
\end{align}
\end{subequations}

We show in Section~\ref{sec:tackle} that the LCO problem can be easily solved without infringing the data privacy of any client.

Assuming that the data instances in the data set $B_i$ are independently and identically drawn from $\mathcal{D}$,
what is the chance that a feasible solution to the LCO problem may achieve a good fairness performance on the distribution $\mathcal{D}$? 

Unfortunately, as indicated in the following analysis, the answer is not encouraging.

\begin{theorem}\label{thm:lco_quality}
Denote by $\sigma_i^2$, $i\in\{1, \ldots, N\}$, the variance of the estimation $\hat L_i^{a,c}(\theta) - \hat L_i^{b,c}(\theta)$ on the $i$-th client, and by $\sigma_{min}^2 = \min (\sigma_1^2, \ldots, \sigma_N^2)$ the minimum variance. 
If the data instances in the data set $B_i$ are independently and identically drawn from $\mathcal{D}$,
then for any feasible solution $\theta$ to the LCO problem and any real number $r > 0$, $P\left(\left|{L^{a,c}(\theta) - L^{b,c}(\theta)}\right| < \epsilon + r\right) \geq  1 - \frac{\sigma_{min}^2}{r^2}$.
\end{theorem}
\begin{proof}
For a client $U_i$, $i\in\{1, \ldots, N\}$, define
%\begin{equation}
$\mu(\theta)=L^{a,c}(\theta) - L^{b,c}(\theta)$
%\end{equation}
and
%\begin{equation}
$R_i(\theta)={\hat L}_i^{a,c}(\theta) - {\hat L}_i^{b,c}(\theta)$,
%\end{equation}
where $\mu(\theta)$ is a random variable because $\theta$ is sampled from the random feasible region defined by Equation~\eqref{simu:oraginab}.

Since all the data instances in the data set $B_i$ are independently and identically drawn from $\mathcal{D}$, $\mathbb{E}(R_i(\theta)) = \mu(\theta)$ always
holds for every sample of $\theta$.
By substituting $\mathbb{E}(R_i(\theta)) = \mu(\theta)$ into Chebyshev's inequality, we have
\begin{equation}
	P\left(\left|{R_i(\theta) - \mu(\theta)}\right| < r\right) \geq 1 - \frac{\sigma_i^2}{r^2}
\end{equation}
for every sample of $\theta$ and any real number $r>0$.

Since a feasible solution $\theta$ to the LCO problem always satisfies $\left|{R_i(\theta)}\right| \leq \epsilon$, 
if $\theta$ satisfies $\left| {R_i(\theta) - \mu(\theta)} \right| < r$, then $\left| {\mu(\theta)} \right| < \epsilon + r$ always holds. 
This further indicates
%\begin{equation}
$P\left(\left|{\mu(\theta)}\right| < \epsilon + r\right) \geq P\left( \left| {R_i(\theta) - \mu(\theta)} \right| < r\right)$,
%\end{equation}
which means
\begin{equation}
	P\left(\left|{\mu(\theta)}\right| < \epsilon + r\right) \geq  1 - \frac{\sigma_i^2}{r^2}.
\end{equation}

Plugging $\mu(\theta)=L^{a,c}(\theta) - L^{b,c}(\theta)$ into the above inequality, we have that the following holds for all $i$ in $\{1, \ldots, N\}$.
%\begin{equation}
$P\left(\left| {L^{a,c}(\theta) - L^{b,c}(\theta)} \right| < \epsilon + r\right) \geq  1 - \frac{\sigma_i^2}{r^2}$.
%\end{equation}	
This gives $N$ independent inequalities that hold for the clients $U_1, \ldots, U_N$, respectively.
In all these $N$ inequalities, the tightest one is
\begin{equation}
	P\left(\left| {L^{a,c}(\theta) - L^{b,c}(\theta)} \right| < \epsilon + r\right) \geq  1 - \frac{\sigma_{min}^2}{r^2},
\end{equation}
because the other $N-1$ inequalities automatically hold when this inequality holds. 
The theorem follows.
\end{proof}

According to Theorem~\ref{thm:lco_quality}, 
when the data instances in the data set $B_i$ are independently and identically drawn from $\mathcal{D}$,
a solution $\theta$ to the LCO problem is $(\epsilon + r)$-fair with a probability no smaller than a lower bound $1 - \frac{\sigma_{min}^2}{r^2}$.

Unfortunately, this lower bound may be small in practice, because the variances $\sigma_1^2, \ldots, \sigma_N^2$ may not be small when the numbers of private data instances owned by the clients are not large enough.
As a result, solving the LCO problem may not have a large probability of achieving a good model fairness performance.

Moreover, as observed in the experiments in Section~\ref{sec:fmtp}, the LCO problem may even be infeasible when a large number of clients introduce a large number of local DGEO constraints. Because the intersection of the feasible regions induced by these constraints may have a good chance to be empty due to the variances $\sigma_1^2, \ldots, \sigma_N^2$.

\subsection{The Federated Estimation Method}\label{sec:tfem}

To tackle the above issues with the local DGEO constraints in the LCO problem, we propose to estimate $L^{a, c}(\theta)  -  L^{b, c}(\theta)$ by 
\begin{equation}\label{eq:fed-est}
	\frac{1}{N}\sum_{i=1}^N \left(\hat L_i^{a,c}(\theta) - \hat L_i^{b,c}(\theta)\right).
\end{equation}

We call this \textbf{federated estimation}
because it is computed by taking the average of the local estimations computed by all the clients.

Similar to how we compute the federated loss, we compute the federated estimation on the server. 
This does not infringe on the data privacy of any client because the local estimations are privately computed by the clients before they are sent to the server.

We use the federated estimation to develop a \textbf{federated DGEO constraint} as
\begin{equation}
	\left|{\frac{1}{N}\sum_{i=1}^N \left(\hat L_i^{a,c}(\theta) - \hat L_i^{b,c}(\theta)\right)}\right| \leq \epsilon,
\end{equation}
and further incorporate this constraint with the federated loss in Equation~\eqref{eq:lf} to convert the  constrained optimization problem in Equation~\eqref{eq:04} into the following \textbf{FedFair problem}.
\begin{subequations}\label{eq:fedfair}
\begin{align}
	&\min_{\theta} \sum_{i=1}^N \frac{m_i}{m_{\mbox{total}}} \hat L_i(\theta) \label{eq:fedfair_a} \\
	&\mbox{s.t. } \left|{\frac{1}{N}\sum_{i=1}^N \left(\hat L_i^{a,c}(\theta) - \hat L_i^{b,c}(\theta)\right)}\right| \le \epsilon \label{eq:fedfair_b}
\end{align}
\end{subequations}

Let us analyze the model fairness performance of a solution to the FedFair problem when the data instances in the data set $B_i$ are independently and identically drawn from $\mathcal{D}$.

\begin{theorem}\label{thm:fedfair}
Denote by $\sigma_i^2$, $i\in\{1, \ldots, N\}$, the variance of the local estimation $\hat L_i^{a,c}(\theta) - \hat L_i^{b,c}(\theta)$ on client $U_i$. 
Let $\sigma_{avg}^2=\frac{1}{N}\sum_{i=1}^N \sigma_i^2$. 
If the data instances in the data set $B_i$ are independently and identically drawn from $\mathcal{D}$,
then for any feasible solution $\theta$ to the FedFair problem and any $r>0$, $P\left(\left|{L^{a,c}(\theta) - L^{b,c}(\theta)}\right| < \epsilon + r\right) \geq  1 - \frac{\sigma_{avg}^2}{Nr^2}$.
\end{theorem}
\begin{proof}
Since the private data sets $B_1, \ldots, B_N$ are independently and identically drawn from $\mathcal{D}$, the local estimations computed on these private data sets are independent of each other.
Therefore, the variance of the federated estimation (see Equation~\eqref{eq:fed-est}) is
\begin{equation}
	\begin{aligned}
		\sigma_{fed}^2
		=Var\left[{\frac{1}{N}\sum_{i=1}^N \left({\hat L}_i^{a,c}(\theta) - {\hat L}_i^{b,c}(\theta)\right)}\right]= \frac{\sigma_{avg}^2}{N}.
	\end{aligned}
\end{equation}

Following the proof of Theorem~\ref{thm:lco_quality} and Chebyshev's inequality, we have, for any $r>0$,
\begin{equation}
	P\left(\left|{L^{a,c}(\theta) - L^{b,c}(\theta)}\right| < \epsilon + r\right) \geq  1 - \frac{\sigma_{fed}^2}{r^2}.
\end{equation}
Since $\sigma_{fed}^2=\frac{\sigma_{avg}^2}{N}$, the theorem follows.
\end{proof}

According to Theorem~\ref{thm:fedfair}, when the data instances in the data set $B_i$ are independently and identically drawn from $\mathcal{D}$,
a feasible solution $\theta$ to the FedFair problem is $(\epsilon + r)$-fair with a probability no smaller than a lower bound $1 - \frac{\sigma_{avg}^2}{Nr^2}$. 
Interestingly, this means that the solution $\theta$ has a higher probability to be $(\epsilon + r)$-fair when the number of clients increases, because the variable $N$ in the lower bound $1 - \frac{\sigma_{avg}^2}{Nr^2}$ is the number of clients.

Based on Theorems~\ref{thm:lco_quality} and~\ref{thm:fedfair}, we can compare the model fairness performance of the FedFair approach and the LCO approach when the data instances in the data set $B_i$ are independently and identically drawn from $\mathcal{D}$.
The key is to analyze the difference between $\sigma_{avg}^2$ and $\sigma_{min}^2$, which depends on the differences among $\sigma_1^2, \ldots, \sigma_N^2$.

Since we use the same formula to compute every local estimation based on the data instances drawn from the same distribution $\mathcal{D}$, the differences among $\sigma_1^2, \ldots, \sigma_N^2$ are determined by the differences among the numbers of data instances owned by different clients.
In typical federated learning applications~\cite{mcmahan2017communication,chen2017distributed,yang2019federated}, clients often have similar amounts of data. 
This produces small differences among $\sigma_1^2, \ldots, \sigma_N^2$, which further means $\sigma_{avg}^2$ and $\sigma_{min}^2$ are close to each other. 

More often than not, the number of clients $N$ is large in typical federated learning scenarios~\cite{kairouz2021advances,yang2019federated}. 
Thus, the variance $\sigma_{fed}^2=\frac{\sigma_{avg}^2}{N}$ of the federated estimation is much smaller than the minimum variance $\sigma_{min}^2$ of the local estimations. 
In consequence, based on Theorems~\ref{thm:lco_quality} and~\ref{thm:fedfair}, a solution to the FedFair problem can have a much higher probability of achieving good model fairness performance than a solution to the LCO problem. 

Theorems~\ref{thm:lco_quality} and~\ref{thm:fedfair} explain the advantage of FedFair over LCO when the data instances in the data set $B_i$ are independently and identically drawn from $\mathcal{D}$.
Although our above analysis is on independently and identically sampled data, as to be discussed later in Section~\ref{sec:exp}, we conduct extensive experiments on data sets that are drawn independently and identically from $\mathcal{D}$ and on data sets that are not drawn independently and identically from $\mathcal{D}$.
FedFair achieves the best performance on all the data sets.

\subsection{Solving the FedFair and LCO Problems}\label{sec:tackle}

In this subsection, we discuss how to solve the FedFair problem and the LCO problem without infringing on the data privacy of any client. 
We first convert the FedFair problem to a nonconvex-concave min-max problem. Then, we apply the alternating gradient projection (AGP) algorithm~\cite{xu2020unified} to 
tackle the min-max problem without infringing the data privacy of any client.
We also extend AGP to tackle the LCO problem in a similar privacy-preserving way.

AGP~\cite{xu2020unified} is designed to obtain a $\delta$-stationary point of a nonconvex-concave minmax problem in $\mathcal{O}(\delta^{-4})$ iterations. It alternatively updates primal and dual variables by employing simple projected gradient steps at each iteration.

Denote by 
$\hat D_i(\theta)= \hat L_i^{a,c}(\theta) - \hat L_i^{b,c}(\theta), i\in\{1, \ldots, N\}$,
the local estimation of client $U_i$.
We rewrite the federated DGEO constraint in the FedFair problem to
%$\frac{1}{N}\sum_{i=1}^N \hat D_i(\theta) \leq \epsilon$, and $- \frac{1}{N}\sum_{i=1}^N \hat D_i(\theta) \leq \epsilon$.
\begin{equation}\label{eq:twocons}
	\begin{aligned}
		-\epsilon \leq \frac{1}{N}\sum_{i=1}^N \hat D_i(\theta) \leq \epsilon,
		%, \text{and }
		%- \frac{1}{N}\sum_{i=1}^N \hat D_i(\theta) \leq \epsilon.
	\end{aligned}
\end{equation}
which represents two constraints on $\theta$.
In this way, the FedFair problem can be equivalently converted to a min-max problem
\begin{equation}\label{minmax}
\min_{\theta}\max_{\lambda_a, \lambda_b \geq 0} {\mathcal L}(\theta, \lambda_a,\lambda_b),
\end{equation}
%where $\lambda_a$ and $\lambda_b$ are Lagrange multipliers for the Lagrangian function as follows.
where
\begin{equation}
\begin{aligned}
     	{\mathcal L}(\theta, \lambda_a, \lambda_b) =  
	&\sum_{i=1}^N \frac{m_i}{m_{\mbox{total}}} \hat L_i(\theta)
	+ \lambda_a \left(\frac{1}{N}\sum_{i=1}^N \hat D_i(\theta) - \epsilon \right) \\
	& -\lambda_b \left(\frac{1}{N}\sum_{i=1}^N \hat D_i(\theta) + \epsilon \right)
\end{aligned}
\end{equation}
is a Lagrangian function with the Lagrange multipliers $\lambda_a$ and $\lambda_b$ introduced for the two constraints in Equation~\eqref{eq:twocons}, respectively.

The Lagrangian function is nonconvex in terms of $\theta$ and concave in terms of $\lambda_a$ and $\lambda_b$.
Thus, the problem in Equation~\eqref{minmax} is a nonconvex-convave min-max problem that can be directly solved using the AGP method~\cite{xu2020unified}.

Each iteration of AGP updates $\theta$, $\lambda_a$, and $\lambda_b$ using the gradients computed from a regularized loss function
\begin{equation}
    \bar{\mathcal L}(\theta, \lambda_a, \lambda_b) = {\mathcal L}(\theta, \lambda_a, \lambda_b) - \frac{\gamma}{2}(\lambda_a^2+\lambda_b^2),
\end{equation}
where $\frac{\gamma}{2}(\lambda_a^2+\lambda_b^2)$ is a regularization term that makes  $\bar{\mathcal L}(\theta, \lambda_a, \lambda_b)$ strongly concave with respect to $\lambda_a$ and $\lambda_b$ to achieve a faster convergence speed, and $\gamma\geq 0$ is a small hyper-parameter to control the level of regularization.

Specifically, the $k$-th iteration of AGP consists of one gradient step to update $\theta$ and one projected gradient descent step to update $\lambda_a$ and $\lambda_b$.

The first step updates $\theta$ with a step size $\alpha$ by 
\begin{equation}
\label{primal-update}
	\theta^{k+1} = \theta^k - {\alpha}\sum_{i=1}^N\nabla_i(\theta^k),
\end{equation}
where
\begin{equation}
\label{eq:gradient-1}
	\nabla_{i}(\theta^k) =
	\frac{m_i}{m_{\mbox{total}}} \nabla \hat L_i({\theta^k} )
	+ \frac{\lambda_a^k - \lambda_b^k}{N}\nabla \hat D_i({\theta^k})
\end{equation}
is the gradient computed by client $U_i$.
Here, $\nabla \hat L_i({\theta^k})$ and $\nabla \hat D_i({\theta^k})$ are respectively the gradients of $\hat L_i(\theta^k)$ and $\hat D_i({\theta^k})$ with respect to $\theta^k$.
These gradients are computed by client $U_i$ using $B_i$.

The second step updates $\lambda_a$ and $\lambda_b$ with a step size $\beta$ by 
\begin{equation}\label{dual-update-1}
	\lambda_{a}^{k+1} = \max\left({(1 - \gamma\beta)\lambda_{a}^{k}
	+ \frac{\beta}{N}\sum_{i=1}^N\hat D_i({\theta^k}) 
	- \beta\epsilon}, 0\right)
\end{equation}
and
\begin{equation}\label{dual-update-2}
	\lambda_{b}^{k+1} = \max\left({(1-\gamma\beta)\lambda_{b}^{k}
	- \frac{\beta}{N}\sum_{i}^N\hat D_i({\theta^k}) 
	- \beta\epsilon}, 0\right).
\end{equation}

The above iteration continues until AGP converges. 
Algorithm~\ref{alg:fedfair} shows how we deploy AGP in a federated learning framework to tackle the FedFair problem.

\begin{algorithm}[t]\small
\DontPrintSemicolon
\SetKwInput{KwInput}{Input}               
\SetKwInput{KwOutput}{Output} 
\KwInput{The clients $U_1, \ldots, U_N$ holding the private data sets $B_1, \ldots, B_N$, respectively; a pre-defined pair of groups $\{a, b\}$; a class $c\in\{1, \ldots, C\}$; the hyperparameters $\alpha$, $\beta$ and $\gamma$ of AGP~\cite{xu2020unified}; and a server $Q$.}
% , local training parameters $\{\beta_{k}\}\subset\R_{++}$}
\KwOutput{A fair model $f_\theta$ that achieves a small DEGO on $\mathcal{D}$ with respect to the groups $\{a, b\}$ and the class $c$.}

$Q$ initializes $\theta^1$, $\lambda_{a}^1$ and $\lambda_{b}^1$ in the same way as AGP~\cite{xu2020unified}.

\For{$k=1,2,\dots,K$}{
	$Q$ broadcasts $\theta^k$, $\lambda_{a}^k$ and $\lambda_{b}^k$ to the clients.

	Each client $U_i$ uses $B_i$ to compute $\hat D_i(\theta^k)$, $\nabla \hat D_i({\theta^k})$ and $\nabla \hat L_i({\theta^k})$, and send them to $Q$.
	
	$Q$ updates $\theta^{k+1}$, $\lambda_a^{k+1}$ and $\lambda_b^{k+1}$ by \eqref{primal-update}, \eqref{dual-update-1} and \eqref{dual-update-2}, respectively. 
}

\Return{The fair model $f_\theta$.}
\caption{Tackling the FedFair Problem}\label{alg:fedfair}
\end{algorithm}
\vspace{-0.13cm}

The LCO problem can be solved in a similar way as the FedFair problem using AGP~\cite{xu2020unified}. 

First, we introduce a pair of Lagrange multipliers, denoted by $\lambda_{a_i}$ and $\lambda_{b_i}$, $i\in\{1, \ldots, N\}$, for each local DGEO constraint rewritten as 
$-\epsilon \leq \hat L_i^{a, c}(\theta) - \hat L_i^{b, c}(\theta) \leq \epsilon$.

Then, we adapt Algorithm~\ref{alg:fedfair} by
replacing Equation~\eqref{eq:gradient-1} by
\begin{equation}
	\nabla_{i}(\theta^k) =
		\frac{m_i}{m_{\mbox{total}}} \nabla \hat L_i({\theta^k} )
		+ (\lambda_{a_i}^k - \lambda_{b_i}^k)\nabla \hat D_i({\theta^k}),
\end{equation}
and replacing Equations~\eqref{dual-update-1} and~\eqref{dual-update-2}, respectively, by
\begin{equation}
	\lambda_{a_i}^{k+1} = \max\left({(1-\gamma\beta)\lambda_{a_i}^{k}
		+ {\beta} \hat D_i({\theta^k}) 
		- \beta\epsilon}, 0\right)
\end{equation}
and
\begin{equation}
	\lambda_{b_i}^{k+1} = \max\left({(1-\gamma\beta)\lambda_{b_i}^{k}
		- {\beta}\hat D_i({\theta^k}) 
		- \beta\epsilon}, 0\right).
\end{equation}

Since both Algorithm~\ref{alg:fedfair} and its extension for the LCO problem implement the AGP algorithm, they have the same convergence behavior as AGP~\cite{xu2020unified}.
The hyper-parameters are set as $\alpha=0.05$, $\beta=0.05$, and $\gamma=0.001$.
To prevent divergence in training, we reduce the value of $\alpha$ by $\alpha \leftarrow 0.1\times \alpha$ per $20,000$ gradient descend steps.

%\textbf{\textit{Does Algorithm~\ref{alg:fedfair} protect the privacy of the private data owned by each client?} }

%The answer is no because 
Algorithm~\ref{alg:fedfair} protects the data privacy of the clients in a similar way as~\cite{mcmahan2017communication,huang2021personalized,li2020federated}. In each iteration, there is no communication between the clients, and the information sent from a client to the server only contains $\hat D_i(\theta^k)$, $\nabla \hat D_i({\theta^k})$, and $\nabla \hat L_i({\theta^k})$, which do not expose the private data sets of any clients or any information about the distribution of the private datasets.

The overall time cost of Algorithm~\ref{alg:fedfair} scales linearly with respect to the number of clients, because 
every client conducts the same computational process and there is no communication between clients.
Comparing to the classic federated learning~\cite{mcmahan2017communication} which does not involve the federated DGEO constraint, the major extra time cost introduced by the federated DGEO constraint is the time for each client to compute and communicate $\hat D_i(\theta^k)$ and $\nabla \hat D_i({\theta^k})$.
Such extra time cost also scales linearly with respect to the number of clients.

%% file: Chapters/exp.tex
% !TEX root = ../main.tex
\section{Experimental Results}\label{sec:exp}
In this section, we conduct a series of experiments to evaluate the performance of our proposed methods for federated fair model training, and compare with the state-of-the-art baselines. We first describe the experiment settings and then present the experimental results.

\begin{figure*}[!htbp]
\centering
\subfigure[\scriptsize {DRUG (IID), LR}]{\includegraphics[width=\figurewidthone]{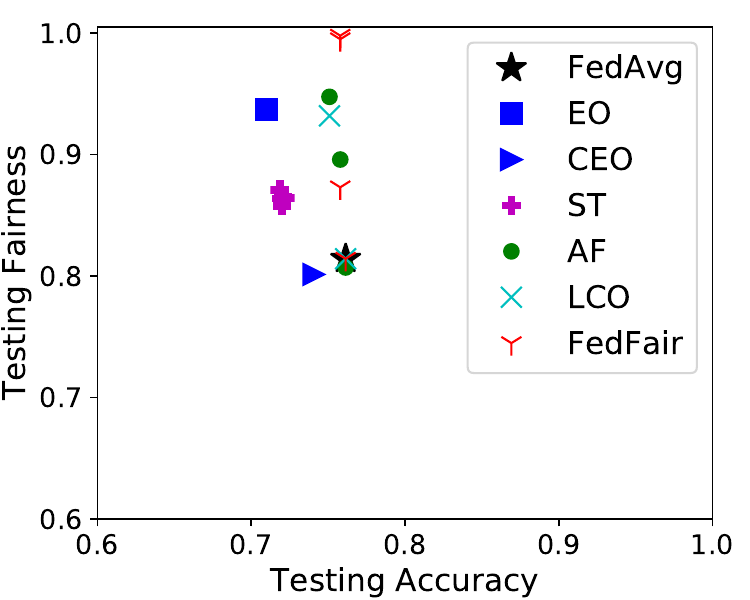}}
\hspace{\figurehspace}
\subfigure[\scriptsize DRUG (IID), NN]{\includegraphics[width=\figurewidthone]{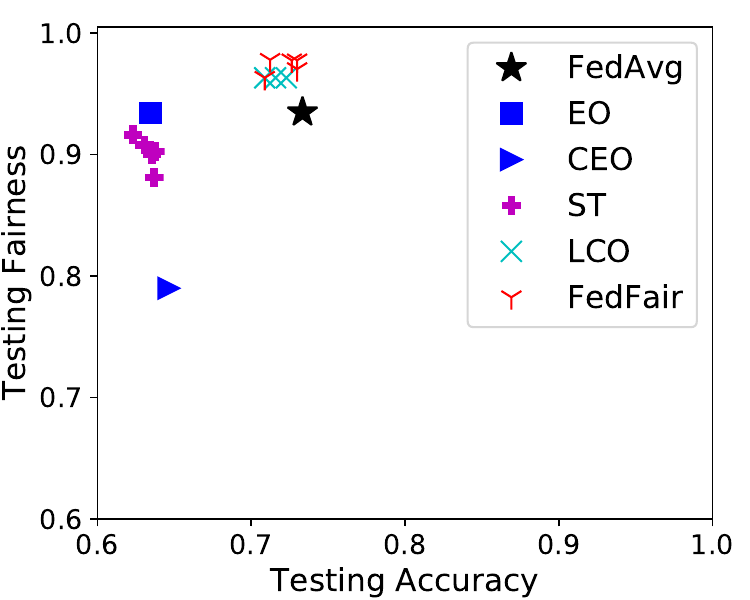}}
\hspace{\figurehspace}
\subfigure[\scriptsize DRUG (non-IID), LR]{\includegraphics[width=\figurewidthone]{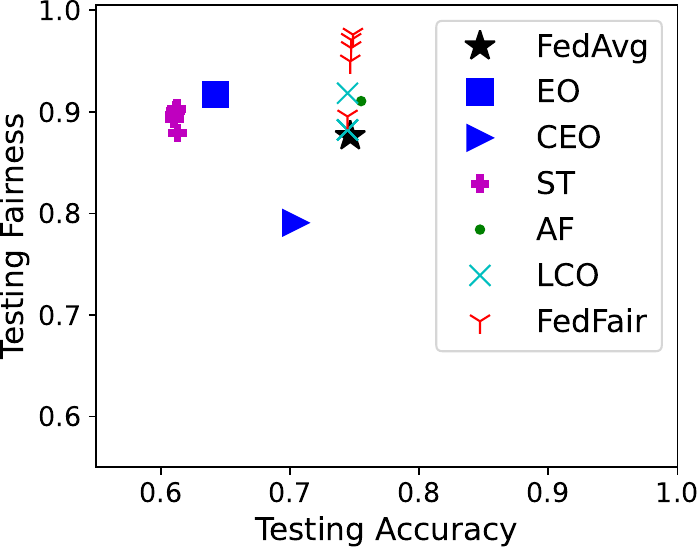}}
\hspace{\figurehspace}
\subfigure[\scriptsize DRUG (non-IID), NN]{\includegraphics[width=\figurewidthone]{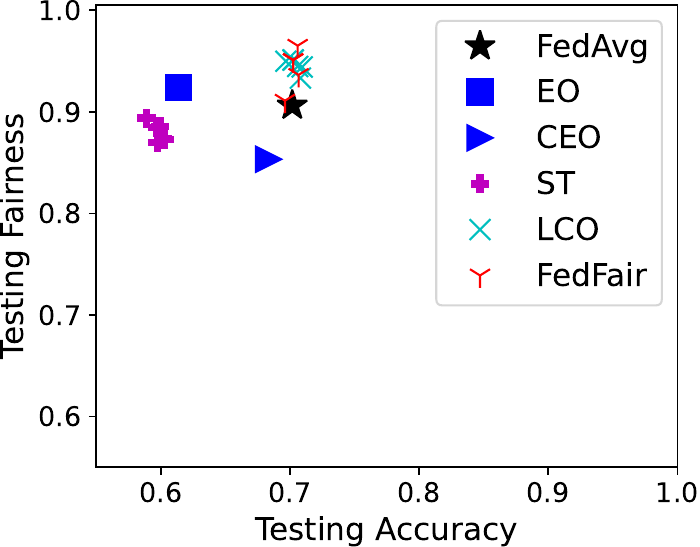}}
%===================================================
\subfigure[\scriptsize COMPASS (IID), LR]{\includegraphics[width=\figurewidthone]{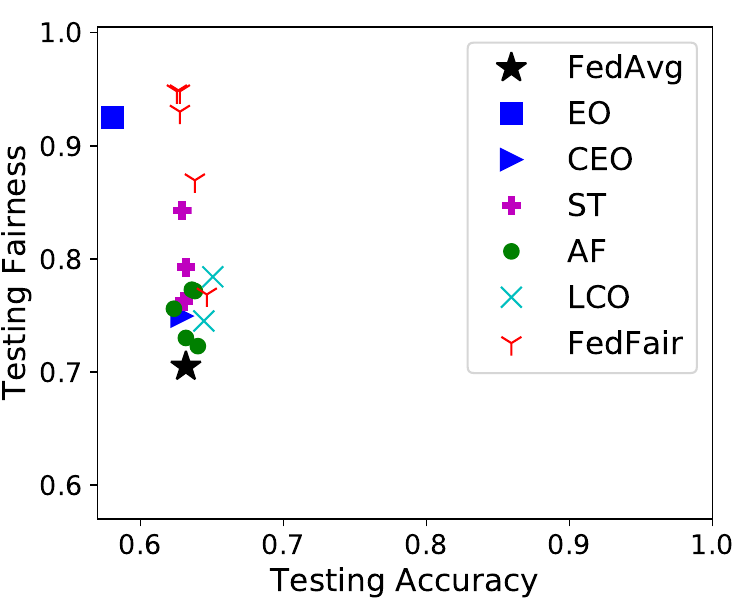}}
\hspace{\figurehspace}
\subfigure[\scriptsize COMPASS (IID), NN]{\includegraphics[width=\figurewidthone]{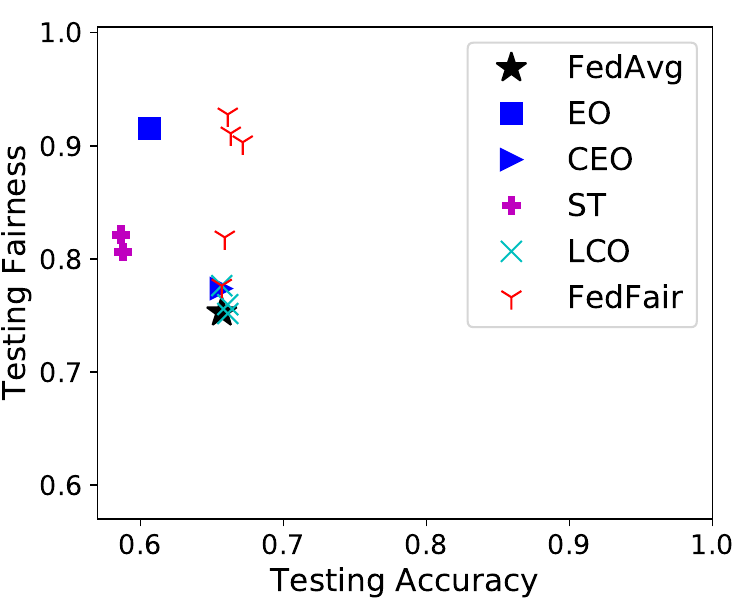}}
\hspace{\figurehspace}
\subfigure[\scriptsize COMPASS (non-IID), LR]{\includegraphics[width=\figurewidthone]{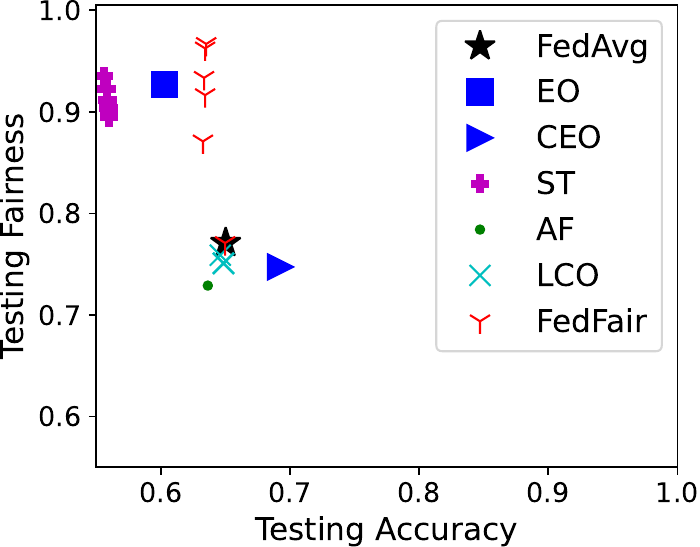}}
\hspace{\figurehspace}
\subfigure[\scriptsize COMPASS (non-IID), NN]{\includegraphics[width=\figurewidthone]{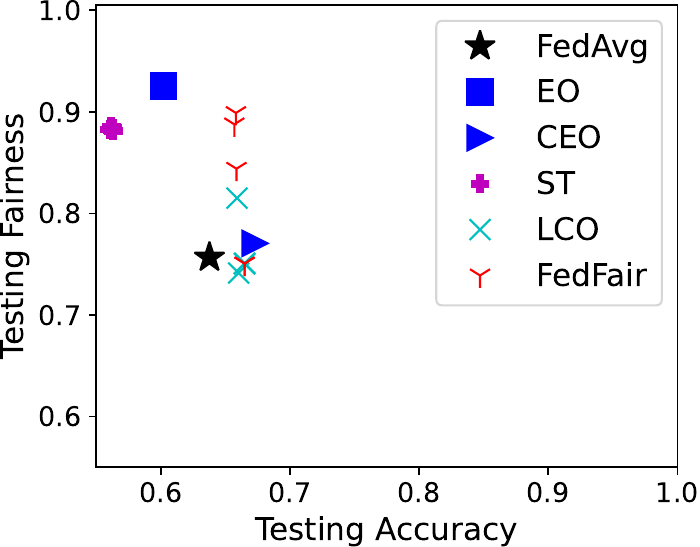}}
%===================================================
\subfigure[\scriptsize ADULT (IID), LR]{\includegraphics[width=\figurewidthone]{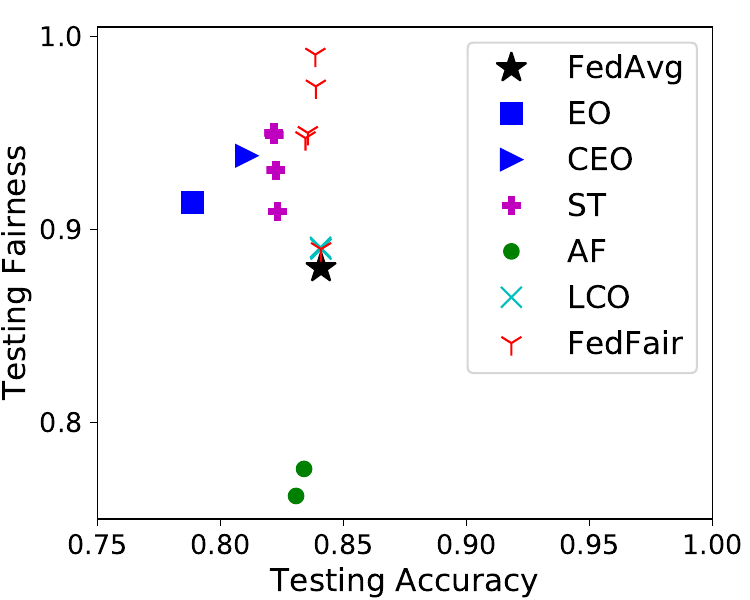}}
\hspace{\figurehspace}
\subfigure[\scriptsize ADULT (IID), NN]{\includegraphics[width=\figurewidthone]{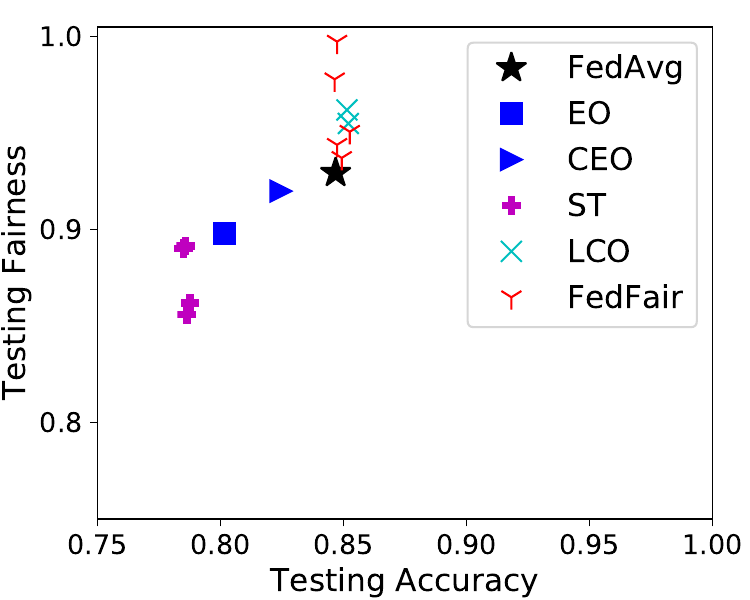}}
\hspace{\figurehspace}
\subfigure[\scriptsize ADULT (non-IID), LR]{\includegraphics[width=\figurewidthone]{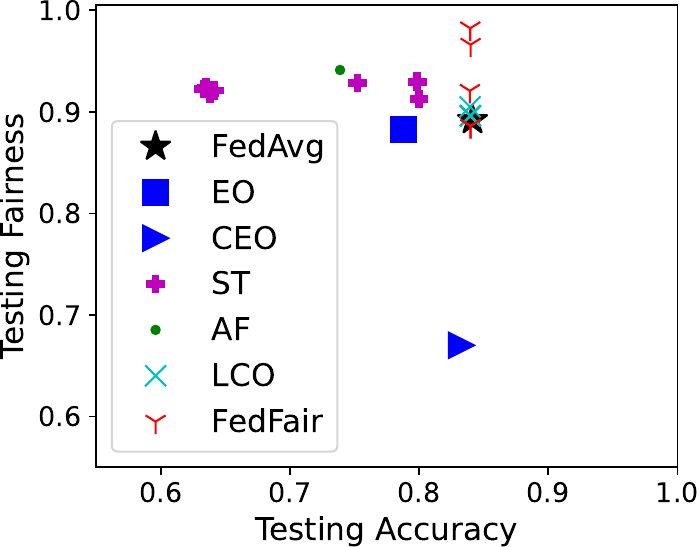}}
\hspace{\figurehspace}
\subfigure[\scriptsize ADULT (non-IID), NN]{\includegraphics[width=\figurewidthone]{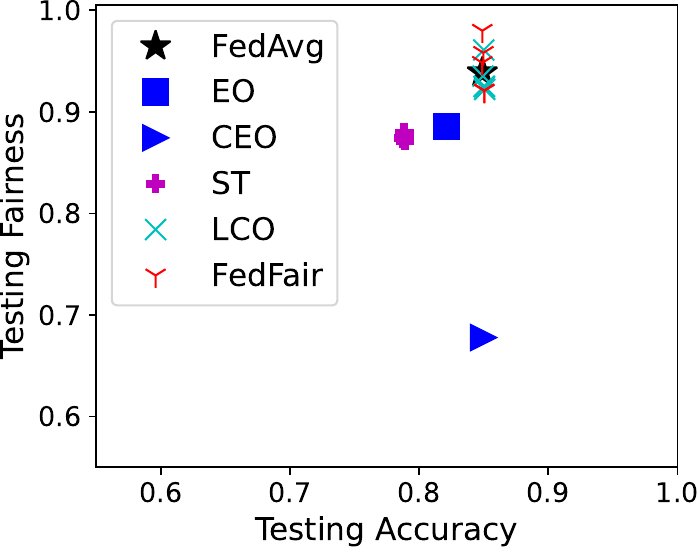}}
%===================================================
\caption{The testing fairness and testing accuracies of all compared methods. 
``LR'' means that the trained model is a logistic regression model; 
``NN'' means that the trained model is a neural network. 
``IID'' and ``non-IID'' refer to the IID setting and non-IID setting of data sets, respectively, as described in Section~\ref{sec:data}.}
\label{fig:overall_perf}
\end{figure*}

\subsection{Experiment Settings}

In our experiments, we evaluate the performance of a series of seven methods.  First of all, FedFair is our proposed method solving the FedFair Problem in Equation~(\ref{eq:fedfair}).  Second, LCO is the method solving the LCO problem in Equation~(\ref{eq:oragin}). Third, we use the state-of-the-art federated fair model training method AgnosticFair (AF)~\cite{du2020fairness} as a baseline.

As the fourth and the fifth methods in our experiments, we compare with a separate training (ST) method and the classical federated learning method FedAvg~\cite{mcmahan2017communication}. We also extend FedAvg by two post-processing fairness enhancing methods, equalized odds (EO)~\cite{hardt2016equality} and calibrated equalized odds (CEO)~\cite{pleiss2017fairness}, to obtain the sixth and the seventh methods in our experiments.

The ST method simply asks each client $U_i$ to use its own private training data set $B_i$ to train a fair model by solving the following constrained optimization problem.
\begin{subequations}
	\begin{align}
    		&\min_{\theta} {\hat L}_i(\theta)\\
    		&\mbox{s.t. } \left| {{\hat L}_i^{a, c}(\theta) - {\hat L}_i^{b, c}(\theta)} \right| \le \epsilon 
	\end{align}
\end{subequations}

The post-processing methods, such as EO and CEO, are used to separately post-process the model trained by FedAvg for each client based on its private training data set. This will produce $N$ fairness enhanced models for $N$ clients. We report the average performance and the variance of those $N$ models.

For all the methods except AF, we use two types of models for $f_\theta$. One is logistic regression, the other one is a fully connected neural network with two hidden layers, where the first and the second hidden layers contain eight and four ReLU units, respectively.

We do not report the performance of AF for the neural network model because the source code of AF is not applicable to training neural networks.

The source codes of all the baseline methods are provided by the authors of~\cite{du2020fairness,mcmahan2017communication,hardt2016equality,pleiss2017fairness}. 
FedFair and LCO are implemented in PyTorch 1.4, the source code is available at: \url{https://bit.ly/fairfed}.
All the experiments are conducted on Dell Alienware with Intel(R) Core(TM) i9-9980XE CPU, 128G memory, NVIDIA 1080Ti, and Ubuntu 16.04.

\subsection{Data Sets}\label{sec:data}
We adopt the following three benchmark data sets that have been widely used in literature to evaluate the performances of fair model training methods~\cite{mehrabi2019survey,donini2018empirical}. 

For each of the following data sets, we use an IID setting and a non-IID setting.
For the \textbf{IID setting}, we uniformly sample the data for every client, such that the data of all clients are independently and identically distributed. 
For the \textbf{non-IID setting}, we adopt the non-IID data partitioning strategy proposed by Huang~\textit{et~al.}~\cite{huang2021personalized}, such that the data of clients are not independently and identically distributed.

The ADULT data set~\cite{kohavi1996scaling} contains 45,222 data instances. Each data instance consists of 14 features of a person. 
The binary class label of a data instance indicates whether a person's annual income is above 50,000 dollars or not. 
Following the settings of Hardt~\textit{et~al.}~\cite{hardt2016equality}, we use ``female'' and ``male'' as the pair of protected groups, and use ``above 50,000 dollars'' as the protected class.
For both the IID setting and the non-IID setting, we use 40,000 data instances as the training data, use 5,222 data instances as the testing data, and use 50 clients each holding $\frac{1}{50}$ of the training data as the private training data set. 
 
The COMPAS data set~\cite{compass_dataset} contains 5,278 data instances. Every data instance consists of 16 features of a person. 
The binary class label of a data instance indicates whether the person is a recidivist or not. 
Following the settings of Hardt~\textit{et~al.}~\cite{hardt2016equality}, we use ``African-American'' and ``Caucasian'' as the pair of protected groups, and use ``not a recidivist'' as the protected class.
For both the IID setting and the non-IID setting, we use 4,800 data instances as the training data set, 478 data instances as the testing data, and 20 clients each holding $\frac{1}{20}$ of the training data as the private training data set.

The DRUG data set~\cite{fehrman2017five} contains 1,885 data instances. Each data instance represents a person characterized by 12 features. 
The binary class label of a data instance indicates whether the person abuses volatile substance or not. 
We use ``white'' and ``non-white'' as the pair of protected groups, and use ``not abuse volatile substance'' as the protected class.
For both the IID setting and the non-IID setting, we use 1,600 data instances as the training data, 285 data instances as the testing data, and 10 clients each holding $\frac{1}{10}$ of the training data as the private training data set. 

For all the data sets, every categorical feature is converted into a set of binary features using one-hot encoding.
% \footnote{\url{https://en.wikipedia.org/wiki/One-hot}}.
The private training data sets owned by the clients are used to train models.
The testing data sets are used to evaluate 
the prediction accuracies and the fairness of the models produced by all the methods.

\begin{table*}[t]
\caption{The performance of all the methods on the data sets in the IID setting. The accuracy, fairness, and harmonic mean are denoted by AC, FR, and HM, respectively. 
``LR'' means that the trained model is a logistic regression model and ``NN'' means that the trained model is a neural network. 
The column OPT reports the optimal hyper-parameter $\tau$ for AF and the optimal hyper-parameter $\epsilon$ for ST, LCO, and FedFair, to achieve the best harmonic mean of fairness and accuracy.}
\label{Table:best_perf_IID}
\resizebox{0.98\textwidth}{!}{%
\begin{NiceTabular}{c|cccc|cccc|cccc}
\hline
\multirow{2}{*}{Method} & \multicolumn{4}{c|}{DRUG (IID), LR} & \multicolumn{4}{c|}{COMPASS (IID), LR} & \multicolumn{4}{c}{ADULT (IID), LR} \\ \cline{2-13} 
 & \multicolumn{1}{c|}{OPT} & \multicolumn{1}{c|}{AC} & \multicolumn{1}{c|}{FR} & HM & \multicolumn{1}{c|}{OPT} & \multicolumn{1}{c|}{AC} & \multicolumn{1}{c|}{FR} & HM & \multicolumn{1}{c|}{OPT} & \multicolumn{1}{c|}{AC} & \multicolumn{1}{c|}{FR} & HM \\ \hline
FedAvg & \multicolumn{1}{c|}{n/a} & \multicolumn{1}{c|}{\textbf{\begin{tabular}[c]{@{}c@{}}0.73$\pm$0.02\end{tabular}}} & \multicolumn{1}{c|}{\begin{tabular}[c]{@{}c@{}}0.88$\pm$0.05\end{tabular}} & \begin{tabular}[c]{@{}c@{}}0.79$\pm$0.02\end{tabular} & \multicolumn{1}{c|}{n/a} & \multicolumn{1}{c|}{\begin{tabular}[c]{@{}c@{}}0.68$\pm$0.02\end{tabular}} & \multicolumn{1}{c|}{\begin{tabular}[c]{@{}c@{}}0.79$\pm$0.05\end{tabular}} & \begin{tabular}[c]{@{}c@{}}0.73$\pm$0.04\end{tabular} & \multicolumn{1}{c|}{n/a} & \multicolumn{1}{c|}{\begin{tabular}[c]{@{}c@{}}0.83$\pm$0.01\end{tabular}} & \multicolumn{1}{c|}{\begin{tabular}[c]{@{}c@{}}0.83$\pm$0.04\end{tabular}} & \begin{tabular}[c]{@{}c@{}}0.83$\pm$0.02\end{tabular} \\ \hline
EO & \multicolumn{1}{c|}{n/a} & \multicolumn{1}{c|}{\begin{tabular}[c]{@{}c@{}}0.70$\pm$0.04\end{tabular}} & \multicolumn{1}{c|}{\begin{tabular}[c]{@{}c@{}}0.92$\pm$0.06\end{tabular}} & \begin{tabular}[c]{@{}c@{}}0.80$\pm$0.04\end{tabular} & \multicolumn{1}{c|}{n/a} & \multicolumn{1}{c|}{\begin{tabular}[c]{@{}c@{}}0.59$\pm$0.02\end{tabular}} & \multicolumn{1}{c|}{\begin{tabular}[c]{@{}c@{}}0.93$\pm$0.03\end{tabular}} & \begin{tabular}[c]{@{}c@{}}0.72$\pm$0.01\end{tabular} & \multicolumn{1}{c|}{n/a} & \multicolumn{1}{c|}{\begin{tabular}[c]{@{}c@{}}0.78$\pm$0.01\end{tabular}} & \multicolumn{1}{c|}{\begin{tabular}[c]{@{}c@{}}0.93$\pm$0.03\end{tabular}} & \begin{tabular}[c]{@{}c@{}}0.85$\pm$0.01\end{tabular} \\ \hline
CEO & \multicolumn{1}{c|}{n/a} & \multicolumn{1}{c|}{\begin{tabular}[c]{@{}c@{}}0.66$\pm$0.04\end{tabular}} & \multicolumn{1}{c|}{\begin{tabular}[c]{@{}c@{}}0.87$\pm$0.02\end{tabular}} & \begin{tabular}[c]{@{}c@{}}0.75$\pm$0.03\end{tabular} & \multicolumn{1}{c|}{n/a} & \multicolumn{1}{c|}{\begin{tabular}[c]{@{}c@{}}0.65$\pm$0.01\end{tabular}} & \multicolumn{1}{c|}{\begin{tabular}[c]{@{}c@{}}0.77$\pm$0.01\end{tabular}} & \begin{tabular}[c]{@{}c@{}}0.70$\pm$0.01\end{tabular} & \multicolumn{1}{c|}{n/a} & \multicolumn{1}{c|}{\begin{tabular}[c]{@{}c@{}}0.80$\pm$0.00\end{tabular}} & \multicolumn{1}{c|}{\begin{tabular}[c]{@{}c@{}}0.96$\pm$0.01\end{tabular}} & \begin{tabular}[c]{@{}c@{}}0.87$\pm$0.00\end{tabular} \\ \hline
ST & \multicolumn{1}{c|}{0.01} & \multicolumn{1}{c|}{\begin{tabular}[c]{@{}c@{}}0.70$\pm$0.02\end{tabular}} & \multicolumn{1}{c|}{\begin{tabular}[c]{@{}c@{}}0.88$\pm$0.02\end{tabular}} & \begin{tabular}[c]{@{}c@{}}0.78$\pm$0.01\end{tabular} & \multicolumn{1}{c|}{0.10} & \multicolumn{1}{c|}{\begin{tabular}[c]{@{}c@{}}0.66$\pm$0.01\end{tabular}} & \multicolumn{1}{c|}{\begin{tabular}[c]{@{}c@{}}0.88$\pm$0.03\end{tabular}} & \begin{tabular}[c]{@{}c@{}}0.75$\pm$0.02\end{tabular} & \multicolumn{1}{c|}{$10^{-4}$} & \multicolumn{1}{c|}{\begin{tabular}[c]{@{}c@{}}0.82$\pm$0.00\end{tabular}} & \multicolumn{1}{c|}{\begin{tabular}[c]{@{}c@{}}0.93$\pm$0.02\end{tabular}} & \begin{tabular}[c]{@{}c@{}}0.87$\pm$0.01\end{tabular} \\ \hline
AF & \multicolumn{1}{c|}{0.05} & \multicolumn{1}{c|}{\textbf{0.73$\pm$0.02}}   & \multicolumn{1}{c|}{0.88$\pm$0.06}   & 0.80$\pm$0.04   & \multicolumn{1}{c|}{0.05} & \multicolumn{1}{c|}{0.68$\pm$0.02}   & \multicolumn{1}{c|}{0.88$\pm$0.08}   & 0.76$\pm$0.04   & \multicolumn{1}{c|}{0.05} & \multicolumn{1}{c|}{0.72$\pm$0.06}   & \multicolumn{1}{c|}{0.83$\pm$0.08}   & 0.77$\pm$0.03   \\ \hline
LCO & \multicolumn{1}{c|}{0.40} & \multicolumn{1}{c|}{0.72$\pm$0.02}   & \multicolumn{1}{c|}{0.90$\pm$0.05}   & 0.80$\pm$0.02   & \multicolumn{1}{c|}{0.48} & \multicolumn{1}{c|}{\textbf{0.69$\pm$0.02}}   & \multicolumn{1}{c|}{0.79$\pm$0.02}   & 0.74$\pm$0.02   & \multicolumn{1}{c|}{0.80} & \multicolumn{1}{c|}{\textbf{0.84$\pm$0.00}}   & \multicolumn{1}{c|}{0.86$\pm$0.04}   & 0.85$\pm$0.02  \\ \hline
FedFair & \multicolumn{1}{c|}{0.07} & \multicolumn{1}{c|}{0.72$\pm$0.02}   & \multicolumn{1}{c|}{\textbf{0.93$\pm$0.04}}   & \textbf{0.81$\pm$0.02}   & \multicolumn{1}{c|}{$10^{-4}$} & \multicolumn{1}{c|}{0.67$\pm$0.03}   & \multicolumn{1}{c|}{\textbf{0.95$\pm$0.02}}   & \textbf{0.79$\pm$0.02}   & \multicolumn{1}{c|}{0.10} & \multicolumn{1}{c|}{0.83$\pm$0.00}   & \multicolumn{1}{c|}{\textbf{0.97$\pm$0.02}}   & \textbf{0.90$\pm$0.01}   \\ \hline
\hline
\multirow{2}{*}{Method} & \multicolumn{4}{c|}{DRUG (IID), NN} & \multicolumn{4}{c|}{COMPASS (IID), NN} & \multicolumn{4}{c}{ADULT (IID), NN} \\ \cline{2-13} 
 & \multicolumn{1}{c|}{OPT} & \multicolumn{1}{c|}{AC} & \multicolumn{1}{c|}{FR} & HM & \multicolumn{1}{c|}{OPT} & \multicolumn{1}{c|}{AC} & \multicolumn{1}{c|}{FR} & HM & \multicolumn{1}{c|}{OPT} & \multicolumn{1}{c|}{AC} & \multicolumn{1}{c|}{FR} & HM \\ \hline
FedAvg                  & \multicolumn{1}{c|}{n/a}   & \multicolumn{1}{c|}{\textbf{0.72$\pm$0.02}}   & \multicolumn{1}{c|}{0.85$\pm$0.07}   & \textbf{0.78$\pm$0.04}   & \multicolumn{1}{c|}{n/a}  & \multicolumn{1}{c|}{\textbf{0.68$\pm$0.02}}   & \multicolumn{1}{c|}{0.81$\pm$0.05}   & 0.74$\pm$0.03   & \multicolumn{1}{c|}{n/a}  & \multicolumn{1}{c|}{0.84$\pm$0.00}   & \multicolumn{1}{c|}{0.87$\pm$0.04}   & 0.85$\pm$0.02   \\ \hline
EO                      & \multicolumn{1}{c|}{n/a}   & \multicolumn{1}{c|}{0.66$\pm$0.07}   & \multicolumn{1}{c|}{\textbf{0.94$\pm$0.05}}   & 0.77$\pm$0.05   & \multicolumn{1}{c|}{n/a}  & \multicolumn{1}{c|}{0.61$\pm$0.02}   & \multicolumn{1}{c|}{0.91$\pm$0.04}   & 0.73$\pm$0.01   & \multicolumn{1}{c|}{n/a}  & \multicolumn{1}{c|}{0.79$\pm$0.01}   & \multicolumn{1}{c|}{0.91$\pm$0.03}   & 0.85$\pm$0.01   \\ \hline
CEO                     & \multicolumn{1}{c|}{n/a}   & \multicolumn{1}{c|}{0.61$\pm$0.04}   & \multicolumn{1}{c|}{0.89$\pm$0.06}   & 0.72$\pm$0.04   & \multicolumn{1}{c|}{n/a}  & \multicolumn{1}{c|}{0.66$\pm$0.00}   & \multicolumn{1}{c|}{0.78$\pm$0.01}   & 0.72$\pm$0.00   & \multicolumn{1}{c|}{n/a}  & \multicolumn{1}{c|}{0.82$\pm$0.00}   & \multicolumn{1}{c|}{0.93$\pm$0.01}   & 0.87$\pm$0.00   \\ \hline
ST                      & \multicolumn{1}{c|}{0.07}  & \multicolumn{1}{c|}{0.63$\pm$0.01}   & \multicolumn{1}{c|}{0.85$\pm$0.04}   & 0.72$\pm$0.02   & \multicolumn{1}{c|}{0.20}     & \multicolumn{1}{c|}{0.60$\pm$0.01}   & \multicolumn{1}{c|}{0.86$\pm$0.02}   & 0.71$\pm$0.01   & \multicolumn{1}{c|}{0.40}     & \multicolumn{1}{c|}{0.79$\pm$0.00}   & \multicolumn{1}{c|}{0.88$\pm$0.02}   & 0.83$\pm$0.01   \\ \hline
LCO                     & \multicolumn{1}{c|}{0.45}  & \multicolumn{1}{c|}{0.70$\pm$0.02}   & \multicolumn{1}{c|}{0.87$\pm$0.09}   & \textbf{0.78$\pm$0.04}   & \multicolumn{1}{c|}{0.60}     & \multicolumn{1}{c|}{\textbf{0.68$\pm$0.02}}   & \multicolumn{1}{c|}{0.81$\pm$0.03}   & 0.74$\pm$0.02   & \multicolumn{1}{c|}{0.80}     & \multicolumn{1}{c|}{\textbf{0.85$\pm$0.00}}   & \multicolumn{1}{c|}{0.92$\pm$0.04}   & 0.89$\pm$0.02   \\ \hline
FedFair                 & \multicolumn{1}{c|}{0.005} & \multicolumn{1}{c|}{0.70$\pm$0.02}   & \multicolumn{1}{c|}{0.87$\pm$0.06}   & \textbf{0.78$\pm$0.04}   & \multicolumn{1}{c|}{0.01}     & \multicolumn{1}{c|}{\textbf{0.68$\pm$0.02}}   & \multicolumn{1}{c|}{\textbf{0.94$\pm$0.04}}   & \textbf{0.79$\pm$0.02}   & \multicolumn{1}{c|}{0.10}     & \multicolumn{1}{c|}{\textbf{0.85$\pm$0.00}}   & \multicolumn{1}{c|}{\textbf{0.98$\pm$0.02}}   & \textbf{0.91$\pm$0.01}   \\ \hline
\end{NiceTabular}%
}
\end{table*}

\begin{table*}[t]
\caption{The performance of all the methods on the data sets in the non-IID setting.
%The notations have the same meaning as in Table~\ref{Table:best_perf_IID}.
The accuracy, fairness, and harmonic mean are denoted by AC, FR, and HM, respectively. 
``LR'' means that the trained model is a logistic regression model and ``NN'' means that the trained model is a neural network. 
The column OPT reports the optimal hyper-parameter $\tau$ for AF and the optimal hyper-parameter $\epsilon$ for ST, LCO, and FedFair, to achieve the best harmonic mean of fairness and accuracy.
}
\label{Table:best_perf_non_IID}
\resizebox{0.98\textwidth}{!}{%
\begin{NiceTabular}{c|cccc|cccc|cccc}
\hline
\multirow{2}{*}{Method} & \multicolumn{4}{c|}{DRUG (non-IID), LR} & \multicolumn{4}{c|}{COMPASS (non-IID), LR} & \multicolumn{4}{c}{ADULT (non-IID), LR} \\ \cline{2-13} 
 & \multicolumn{1}{c|}{OPT} & \multicolumn{1}{c|}{AC} & \multicolumn{1}{c|}{FR} & HM & \multicolumn{1}{c|}{OPT} & \multicolumn{1}{c|}{AC} & \multicolumn{1}{c|}{FR} & HM & \multicolumn{1}{c|}{OPT} & \multicolumn{1}{c|}{AC} & \multicolumn{1}{c|}{FR} & HM \\ \hline
FedAvg                  & \multicolumn{1}{c|}{n/a}   & \multicolumn{1}{c|}{0.75$\pm$0.01}  & \multicolumn{1}{c|}{0.88$\pm$0.06}  & 0.81$\pm$0.03  & \multicolumn{1}{c|}{n/a}  & \multicolumn{1}{c|}{0.65$\pm$0.01}  & \multicolumn{1}{c|}{0.77$\pm$0.02}  & 0.71$\pm$0.01  & \multicolumn{1}{c|}{n/a}  & \multicolumn{1}{c|}{\textbf{0.84$\pm$0.00}} & \multicolumn{1}{c|}{0.89$\pm$0.01} & 0.87$\pm$0.01 \\ \hline
EO                      & \multicolumn{1}{c|}{n/a}   & \multicolumn{1}{c|}{0.64$\pm$0.09}  & \multicolumn{1}{c|}{0.92$\pm$0.05}  & 0.76$\pm$0.07  & \multicolumn{1}{c|}{n/a}  & \multicolumn{1}{c|}{0.60$\pm$0.02}  & \multicolumn{1}{c|}{0.93$\pm$0.03}  & 0.73$\pm$0.01  & \multicolumn{1}{c|}{n/a}  & \multicolumn{1}{c|}{0.79$\pm$0.07} & \multicolumn{1}{c|}{0.88$\pm$0.05} & 0.83$\pm$0.05 \\ \hline
CEO                     & \multicolumn{1}{c|}{n/a}   & \multicolumn{1}{c|}{0.70$\pm$0.00}   & \multicolumn{1}{c|}{0.79$\pm$0.02}   & 0.74$\pm$0.01   & \multicolumn{1}{c|}{n/a}  & \multicolumn{1}{c|}{\textbf{0.69$\pm$0.01}}   & \multicolumn{1}{c|}{0.75$\pm$0.01}   & 0.72$\pm$0.01   & \multicolumn{1}{c|}{n/a}  & \multicolumn{1}{c|}{0.83$\pm$0.00}   & \multicolumn{1}{c|}{0.67$\pm$0.00}   & 0.74$\pm$0.00   \\ \hline
ST                      & \multicolumn{1}{c|}{0.30}  & \multicolumn{1}{c|}{0.61$\pm$0.01}   & \multicolumn{1}{c|}{0.90$\pm$0.01}   & 0.73$\pm$0.01   & \multicolumn{1}{c|}{0.10} & \multicolumn{1}{c|}{0.56$\pm$0.00}   & \multicolumn{1}{c|}{0.94$\pm$0.01}   & 0.70$\pm$0.00   & \multicolumn{1}{c|}{0.01} & \multicolumn{1}{c|}{0.80$\pm$0.00}   & \multicolumn{1}{c|}{0.93$\pm$0.01}   & 0.86$\pm$0.00   \\ \hline
AF                      & \multicolumn{1}{c|}{1.00}  & \multicolumn{1}{c|}{\textbf{0.76$\pm$0.01}}   & \multicolumn{1}{c|}{0.91$\pm$0.06}   & 0.82$\pm$0.02   & \multicolumn{1}{c|}{0.50} & \multicolumn{1}{c|}{0.64$\pm$0.01}   & \multicolumn{1}{c|}{0.73$\pm$0.04}   & 0.68$\pm$0.02   & \multicolumn{1}{c|}{2.50} & \multicolumn{1}{c|}{0.74$\pm$0.01}   & \multicolumn{1}{c|}{0.94$\pm$0.03}   & 0.83$\pm$0.02   \\ \hline
LCO                     & \multicolumn{1}{c|}{0.40}  & \multicolumn{1}{c|}{0.74$\pm$0.01}   & \multicolumn{1}{c|}{0.92$\pm$0.07}   & 0.82$\pm$0.03   & \multicolumn{1}{c|}{0.50} & \multicolumn{1}{c|}{0.65$\pm$0.01}   & \multicolumn{1}{c|}{0.76$\pm$0.02}   & 0.70$\pm$0.01   & \multicolumn{1}{c|}{1.00} & \multicolumn{1}{c|}{\textbf{0.84$\pm$0.00}}   & \multicolumn{1}{c|}{0.90$\pm$0.02}   & 0.87$\pm$0.01   \\ \hline
FedFair                 & \multicolumn{1}{c|}{0.02}  & \multicolumn{1}{c|}{0.75$\pm$0.01}   & \multicolumn{1}{c|}{\textbf{0.98$\pm$0.02}}   & \textbf{0.85$\pm$0.01}   & \multicolumn{1}{c|}{0.005} & \multicolumn{1}{c|}{0.64$\pm$0.01}   & \multicolumn{1}{c|}{\textbf{0.97$\pm$0.04}}   & \textbf{0.77$\pm$0.01}   & \multicolumn{1}{c|}{0.10} & \multicolumn{1}{c|}{\textbf{0.84$\pm$0.00}}   & \multicolumn{1}{c|}{\textbf{0.98$\pm$0.01}}   & \textbf{0.91$\pm$0.00}   \\ \hline		
\hline
\multirow{2}{*}{Method} & \multicolumn{4}{c|}{DRUG (non-IID), NN} & \multicolumn{4}{c|}{COMPASS (non-IID), NN} & \multicolumn{4}{c}{ADULT (non-IID), NN} \\ \cline{2-13} 
 & \multicolumn{1}{c|}{OPT} & \multicolumn{1}{c|}{AC} & \multicolumn{1}{c|}{FR} & HM & \multicolumn{1}{c|}{OPT} & \multicolumn{1}{c|}{AC} & \multicolumn{1}{c|}{FR} & HM & \multicolumn{1}{c|}{OPT} & \multicolumn{1}{c|}{AC} & \multicolumn{1}{c|}{FR} & HM \\ \hline
FedAvg                  & \multicolumn{1}{c|}{n/a}   & \multicolumn{1}{c|}{0.70$\pm$0.02}   & \multicolumn{1}{c|}{0.91$\pm$0.05}   & 0.79$\pm$0.02   & \multicolumn{1}{c|}{n/a}  & \multicolumn{1}{c|}{0.64$\pm$0.02}   & \multicolumn{1}{c|}{0.76$\pm$0.05}   & 0.69$\pm$0.02   & \multicolumn{1}{c|}{n/a}  & \multicolumn{1}{c|}{\textbf{0.85$\pm$0.00}}   & \multicolumn{1}{c|}{0.94$\pm$0.02}   & 0.89$\pm$0.01   \\ \hline
EO                      & \multicolumn{1}{c|}{n/a}   & \multicolumn{1}{c|}{0.61$\pm$0.06}   & \multicolumn{1}{c|}{0.92$\pm$0.04}   & 0.74$\pm$0.05   & \multicolumn{1}{c|}{n/a}  & \multicolumn{1}{c|}{0.60$\pm$0.01}   & \multicolumn{1}{c|}{\textbf{0.93$\pm$0.04}}   & 0.73$\pm$0.01   & \multicolumn{1}{c|}{n/a}  & \multicolumn{1}{c|}{0.82$\pm$0.07}   & \multicolumn{1}{c|}{0.89$\pm$0.05}   & 0.85$\pm$0.04   \\ \hline
CEO                     & \multicolumn{1}{c|}{n/a}   & \multicolumn{1}{c|}{0.68$\pm$0.02}   & \multicolumn{1}{c|}{0.84$\pm$0.08}   & 0.75$\pm$0.02   & \multicolumn{1}{c|}{n/a}  & \multicolumn{1}{c|}{\textbf{0.67$\pm$0.02}}   & \multicolumn{1}{c|}{0.77$\pm$0.03}   & 0.72$\pm$0.02   & \multicolumn{1}{c|}{n/a}  & \multicolumn{1}{c|}{\textbf{0.85$\pm$0.00}}   & \multicolumn{1}{c|}{0.68$\pm$0.00}   & 0.75$\pm$0.00   \\ \hline
ST                      & \multicolumn{1}{c|}{0.50}  & \multicolumn{1}{c|}{0.60$\pm$0.02}   & \multicolumn{1}{c|}{0.89$\pm$0.01}   & 0.71$\pm$0.01   & \multicolumn{1}{c|}{0.05}     & \multicolumn{1}{c|}{0.56$\pm$0.00}   & \multicolumn{1}{c|}{0.88$\pm$0.01}   & 0.69$\pm$0.00   & \multicolumn{1}{c|}{0.40}     & \multicolumn{1}{c|}{0.79$\pm$0.00}   & \multicolumn{1}{c|}{0.88$\pm$0.02}   & 0.83$\pm$0.01   \\ \hline
LCO                     & \multicolumn{1}{c|}{0.80}  & \multicolumn{1}{c|}{0.71$\pm$0.03}   & \multicolumn{1}{c|}{0.94$\pm$0.04}   & 0.81$\pm$0.03   & \multicolumn{1}{c|}{0.20}     & \multicolumn{1}{c|}{0.66$\pm$0.01}   & \multicolumn{1}{c|}{0.82$\pm$0.04}   & 0.73$\pm$0.02   & \multicolumn{1}{c|}{0.30}     & \multicolumn{1}{c|}{\textbf{0.85$\pm$0.00}}   & \multicolumn{1}{c|}{0.96$\pm$0.02}   & 0.90$\pm$0.01   \\ \hline
FedFair                 & \multicolumn{1}{c|}{0.10} & \multicolumn{1}{c|}{\textbf{0.71$\pm$0.02}}   & \multicolumn{1}{c|}{\textbf{0.97$\pm$0.03}}   & \textbf{0.82$\pm$0.02}   & \multicolumn{1}{c|}{0.01}     & \multicolumn{1}{c|}{0.66$\pm$0.01}   & \multicolumn{1}{c|}{0.90$\pm$0.04}   & \textbf{0.76$\pm$0.01}   & \multicolumn{1}{c|}{0.10}     & \multicolumn{1}{c|}{\textbf{0.85$\pm$0.00}}   & \multicolumn{1}{c|}{\textbf{0.98$\pm$0.02}}   & \textbf{0.91$\pm$0.01}   \\ \hline
\end{NiceTabular}%
}
\end{table*}

\subsection{Fairness Measure}\label{sec:fm}

In our experiments, the \textbf{fairness} of a trained model $f_\theta$ is evaluated by
\begin{equation}
\label{eq:fairmetric}
    \text{fairness}=1 - \text{DEO}(f_\theta),    
\end{equation}
where $\text{DEO}(f_\theta)$ is the difference of equal opportunity (DEO)~\cite{donini2018empirical} of the trained model $f_\theta$, and it is computed by
\begin{equation}
    \begin{aligned}
        \text{DEO}(f_\theta)=
        &|P\{f_\theta(x) > 0 \mid y = c, s=a\} - \\
        & P\{f_\theta(x) > 0 \mid y = c, s=b\}|.
    \end{aligned}
\end{equation}

Essentially, for an instance $x$ with a ground truth label $y=c$, $c$ is the protected class,  
$\text{DEO}(f_\theta)$ measures the absolute difference of the probabilities of $x$ being predicted by $f_\theta$ to be in the class $c$ when $x$ belongs to the groups $a$ and $b$, respectively.
A smaller value of $\text{DEO}(f_\theta)$ means $f_\theta$ is more fair and a larger value means $f_\theta$ is less fair.

According to~\cite{donini2018empirical}, $\text{DEO}(f_\theta)$ is equivalent to DGEO when the loss function is a discrete function, that is,
$\ell(f_\theta(\mathbf{x}), y) = \mathbb{1}_{y*f_\theta(\mathbf{x})\leq 0}$.
The value range of $\text{DEO}(f_\theta)$ is $[0,1]$, but the value range of DGEO is unbounded. Therefore, $\text{DEO}(f_\theta)$ is a better choice to develop an evaluation metric of fairness in our experiments.
As a result, we use $\text{DEO}(f_\theta)$ to develop the ``fairness'' in Equation~\eqref{eq:fairmetric}. Obviously, ``fairness'' also has a value range of $[0,1]$ and a larger value of fairness means $f_\theta$ is more fair.

We also evaluate the overall performance of $f_\theta$ by the harmonic mean of its fairness and accuracy, which is simply referred to as the \textbf{harmonic mean (HM)} when the context is clear.
For ST, EO and CEO that produce $N$ different models for $N$ clients, respectively, we report the average accuracy, the average fairness, and the average of the harmonic means of the $N$ models on the testing data sets.

\subsection{Fairness and Accuracy of Trained Models}\label{sec:fmtp}
Each of ST, LCO, AF, and FedFair has some hyper-parameters that control the trade-off between fairness and accuracy. 
To comprehensively compare the performance of these methods, we conduct a grid search on the hyper-parameters.
This produces multiple trained models for every method, where each model is trained using a unique value of the hyper-parameter.

For every method, we report the performance of the models that achieve the top-5 largest harmonic means.
For each trained model, the fairness and the accuracy are plotted as a single point in Figure~\ref{fig:overall_perf}.
Since we train 5 models for each method, every method has 5 points in Figure~\ref{fig:overall_perf}.
Some methods show less than five points in a figure because there are some of those points overlap with each other.

We also report the performance of FedAvg, EO, and CEO in Figure~\ref{fig:overall_perf}. 
However, since these methods do not provide a hyper-parameter to control the trade-off between fairness and accuracy, we can only use their default settings to run each of them once. As a result, there is only one point for each of FedAvg, EO, and CEO in every figure of Figure~\ref{fig:overall_perf}.

\begin{figure*}[t]
\centering
%===================================================
\subfigure[\scriptsize DRUG (IID), LR]{\includegraphics[width=\figurewidthtwo]{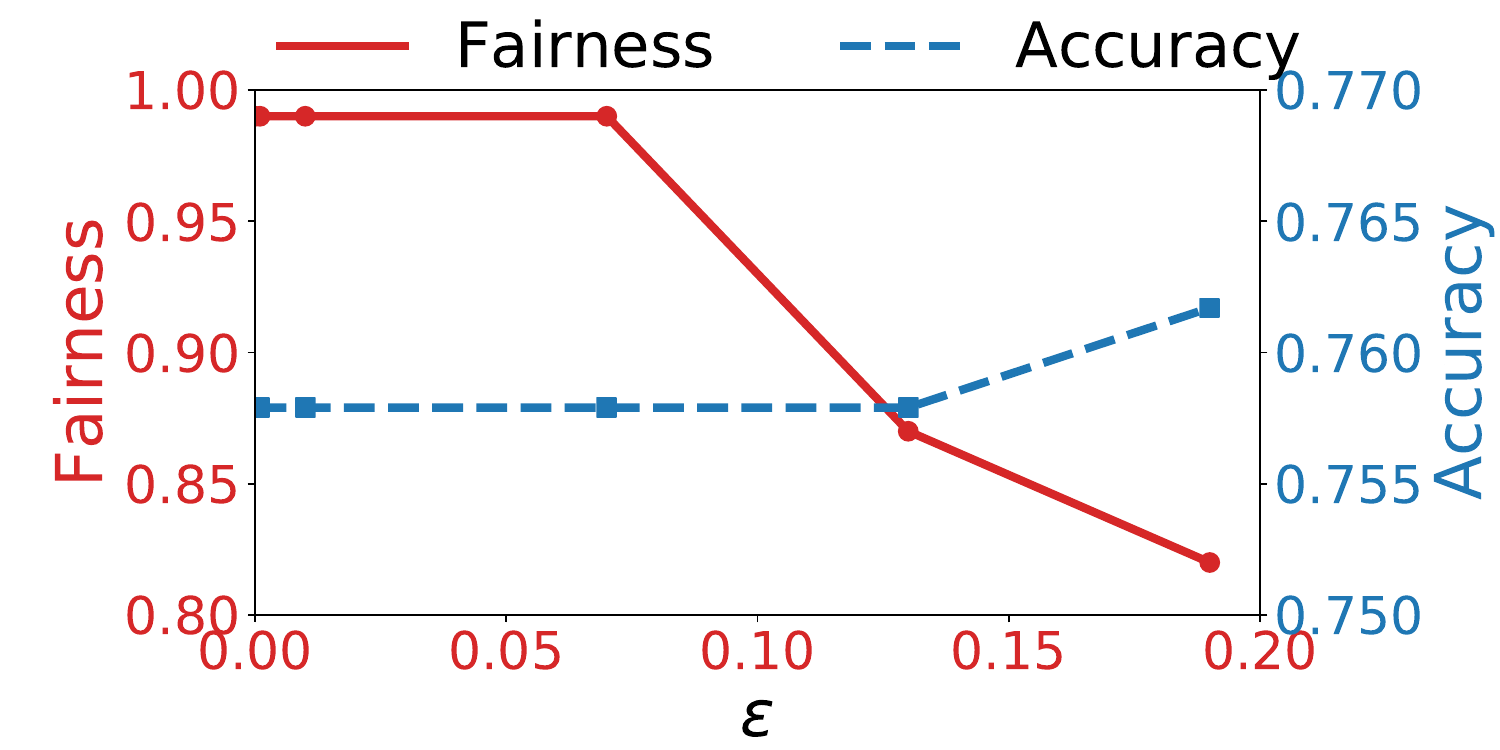}}
\hspace{\figurehspace}
\subfigure[\scriptsize DRUG (IID), NN]{\includegraphics[width=\figurewidthtwo]{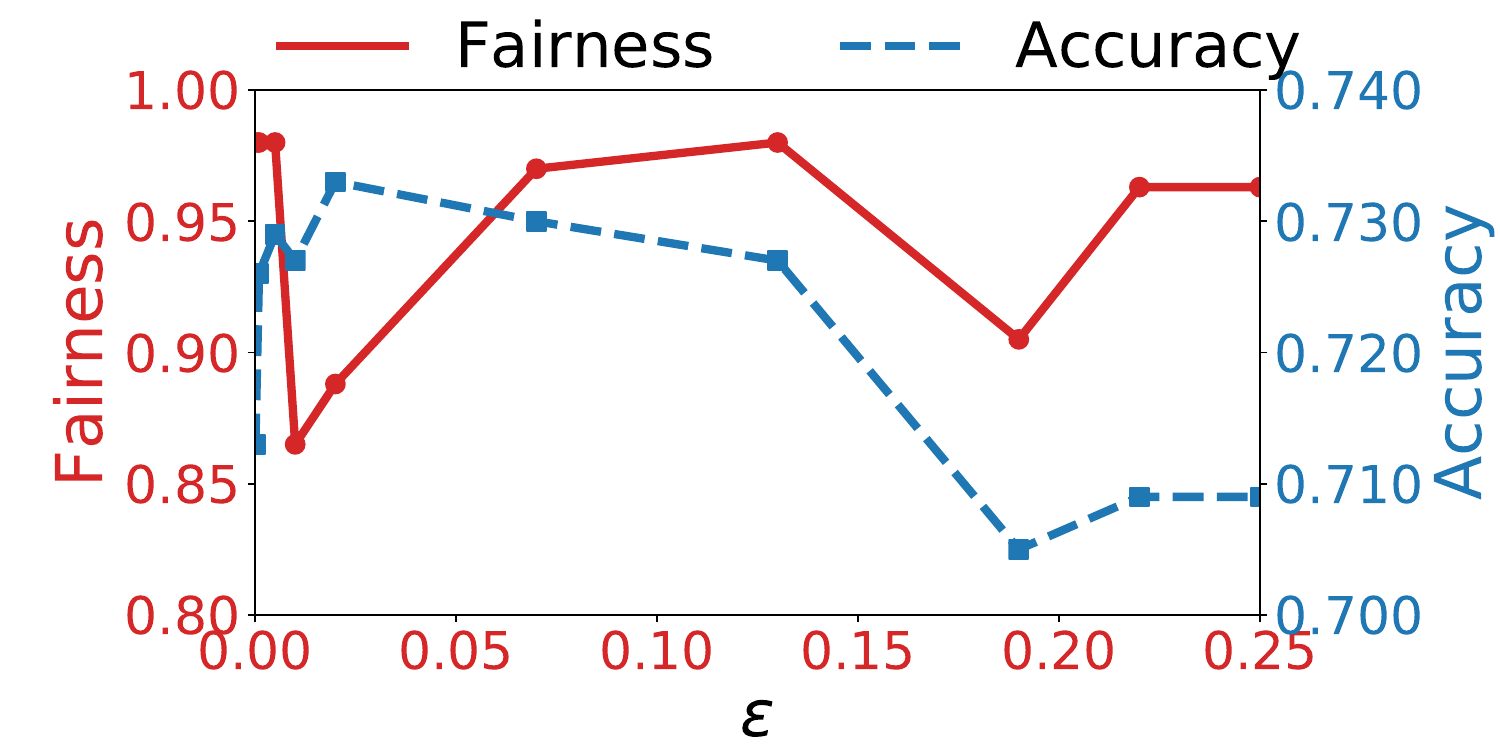}}
\hspace{\figurehspace}
\subfigure[\scriptsize DRUG (non-IID), LR]{\includegraphics[width=\figurewidthtwo]{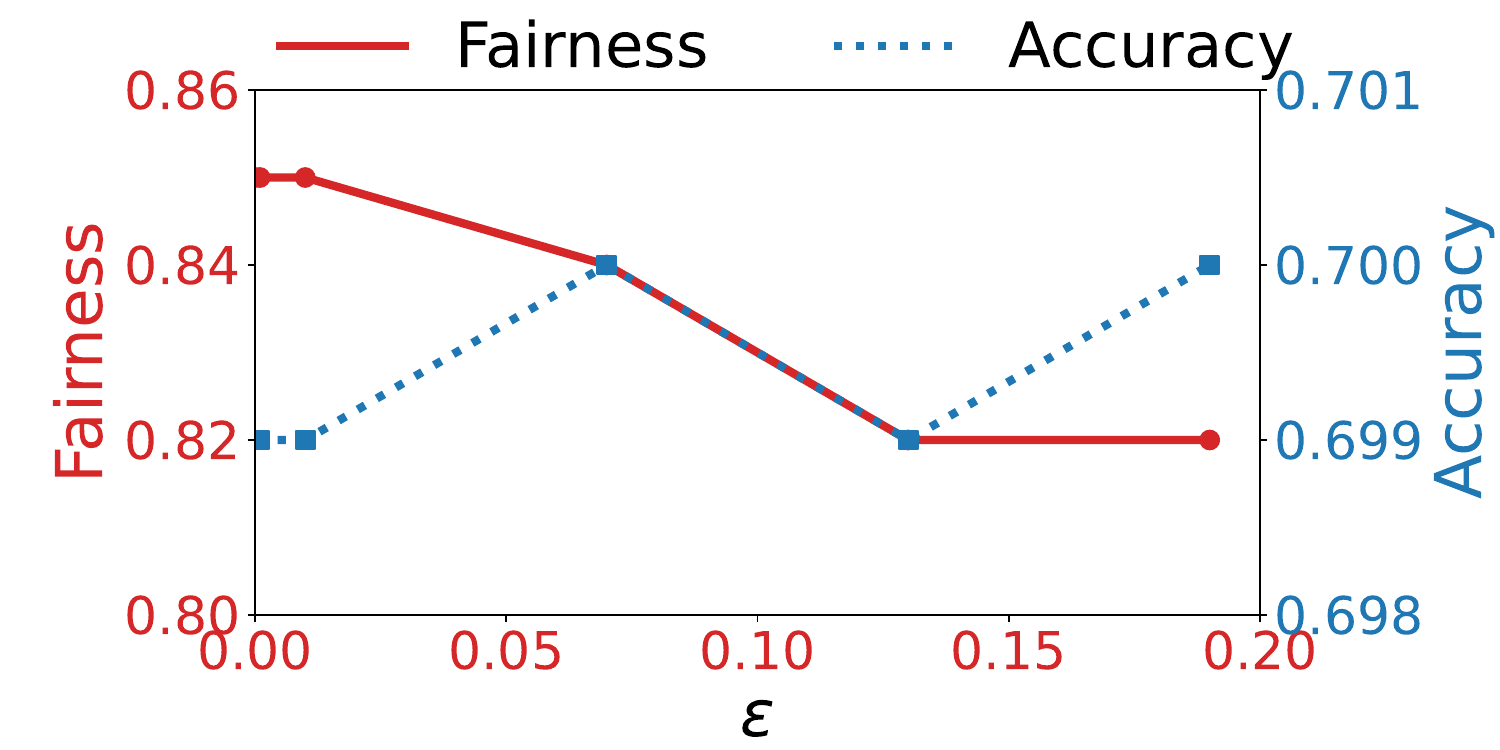}}
\hspace{\figurehspace}
\subfigure[\scriptsize DRUG (non-IID), NN]{\includegraphics[width=\figurewidthtwo]{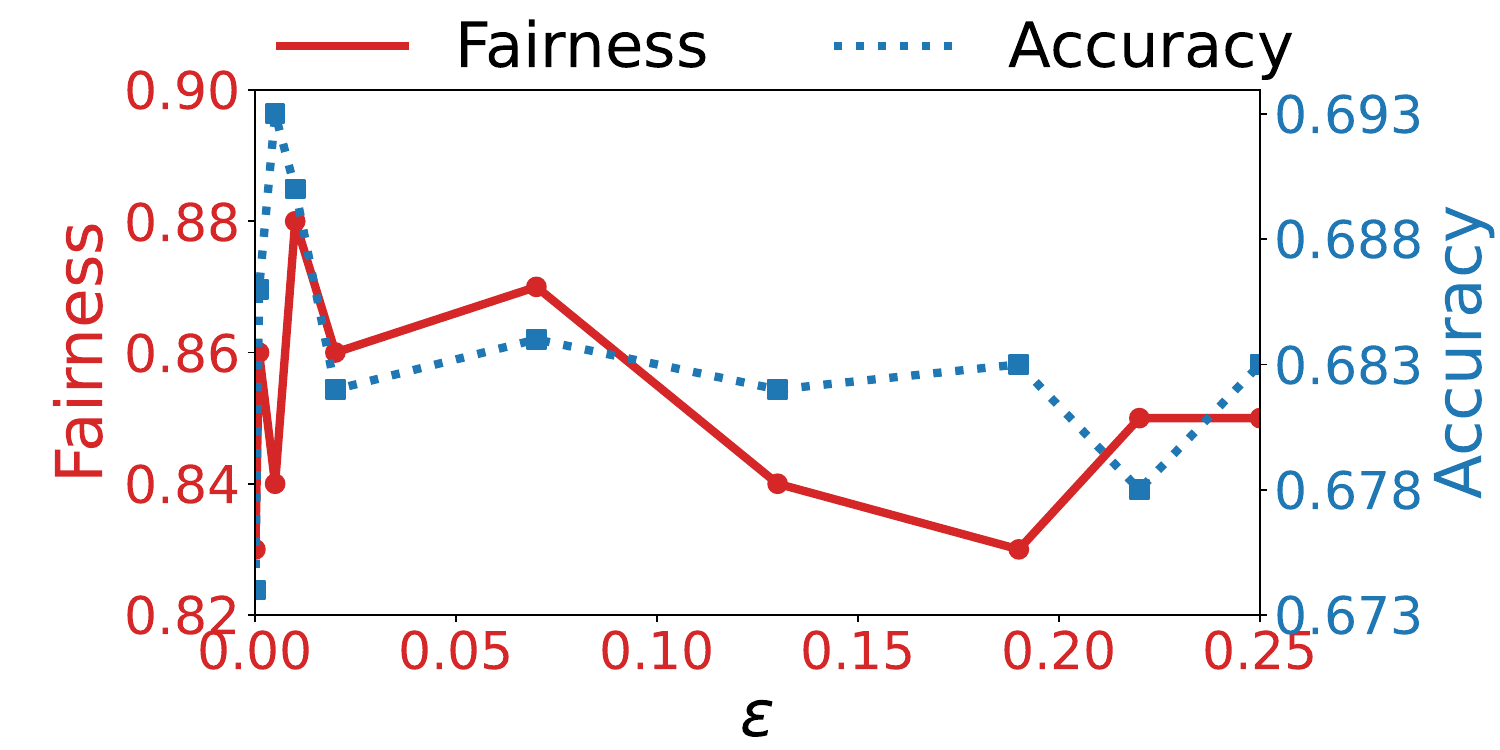}}
%===================================================
\subfigure[\scriptsize COMPAS (IID), LR]{\includegraphics[width=\figurewidthtwo]{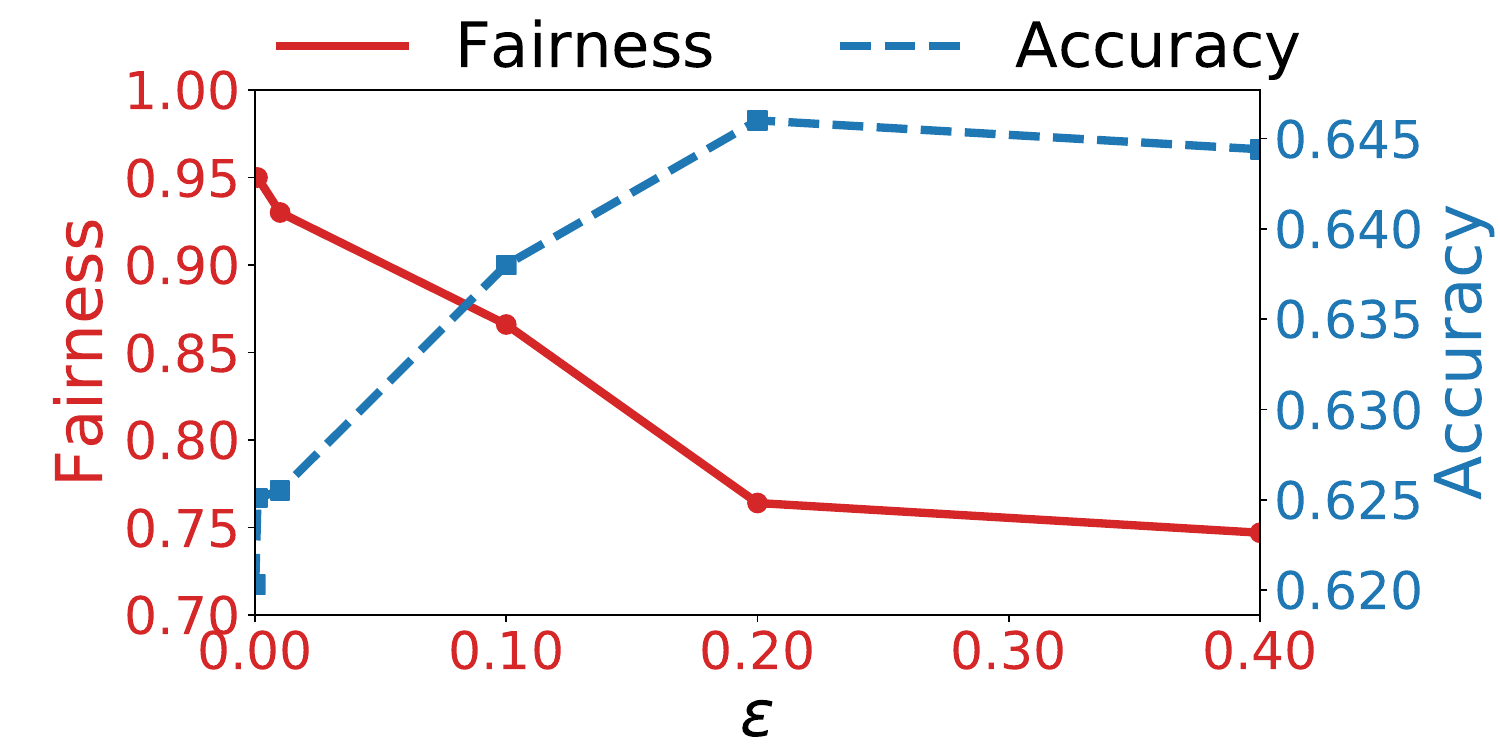}}
\hspace{\figurehspace}
\subfigure[\scriptsize COMPAS (IID), NN]{\includegraphics[width=\figurewidthtwo]{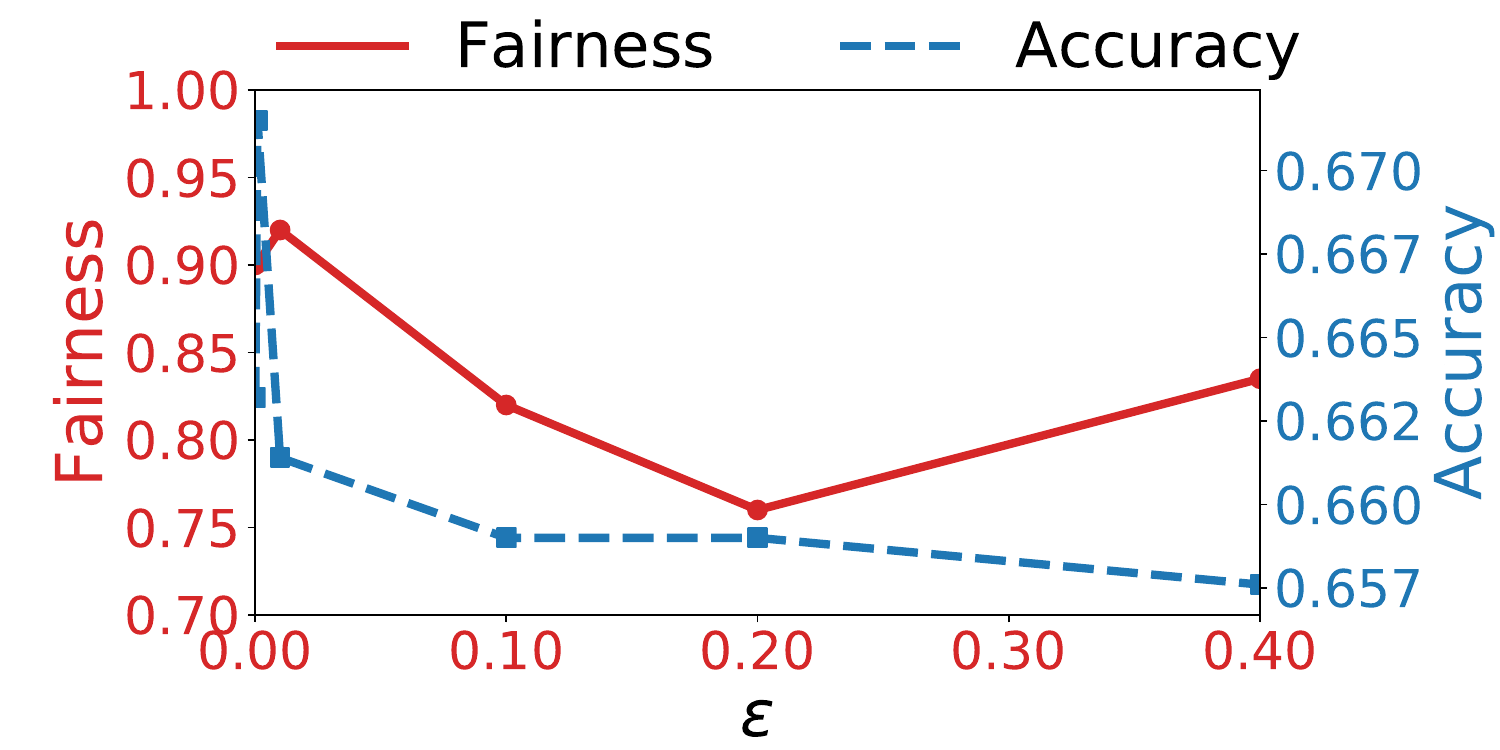}}
\hspace{\figurehspace}
\subfigure[\scriptsize COMPAS (non-IID), LR]{\includegraphics[width=\figurewidthtwo]{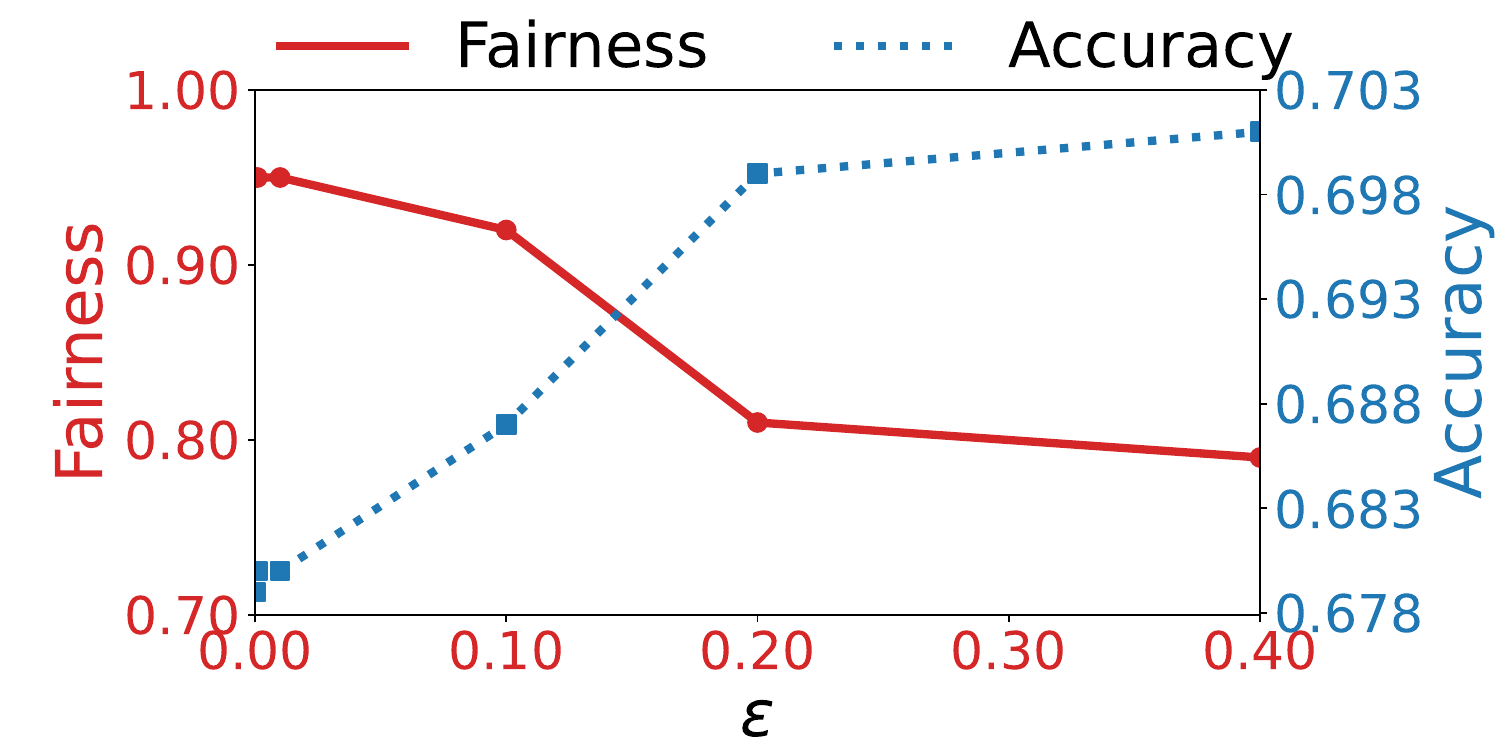}}
\hspace{\figurehspace}
\subfigure[\scriptsize COMPAS(non-IID), NN]{\includegraphics[width=\figurewidthtwo]{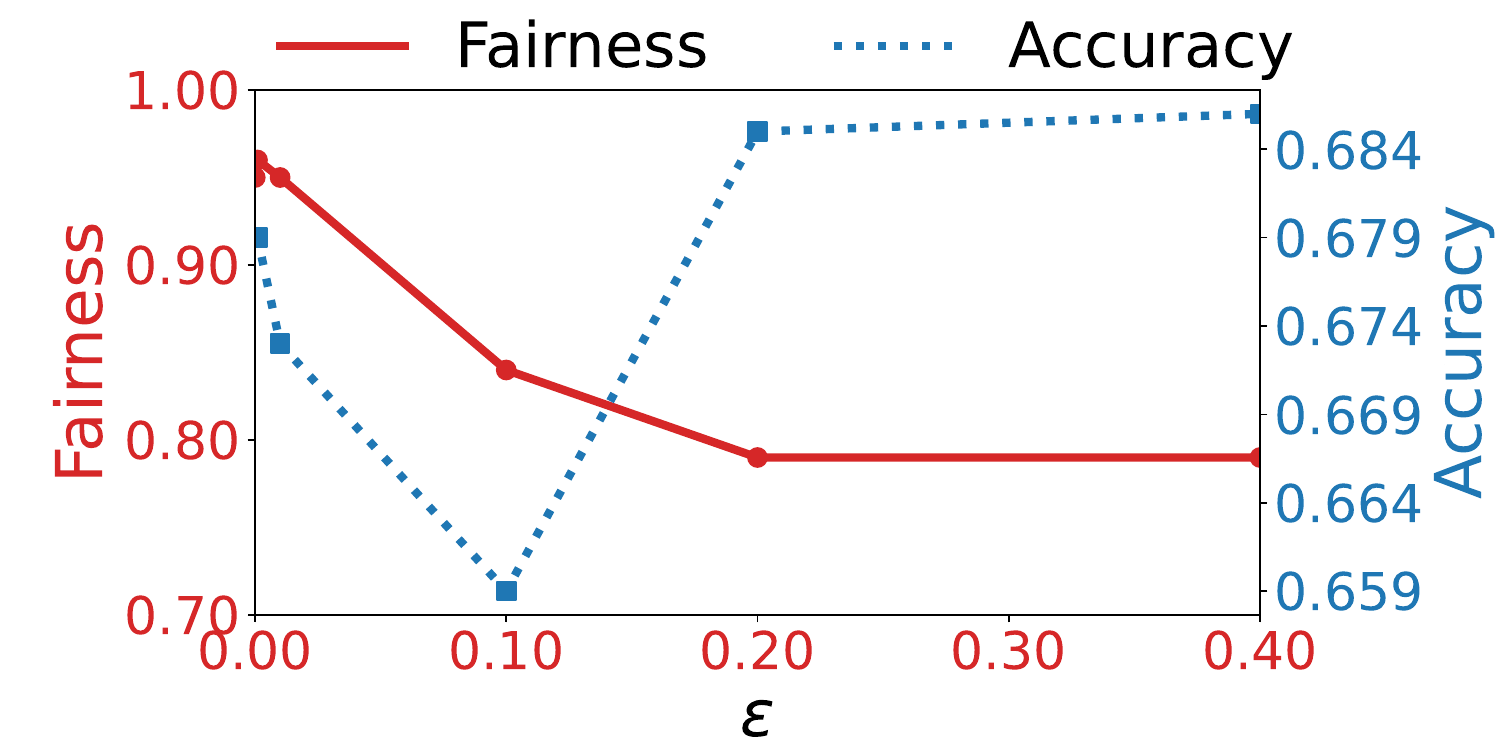}}
%===================================================
\subfigure[\scriptsize ADULT (IID), LR]{\includegraphics[width=\figurewidthtwo]{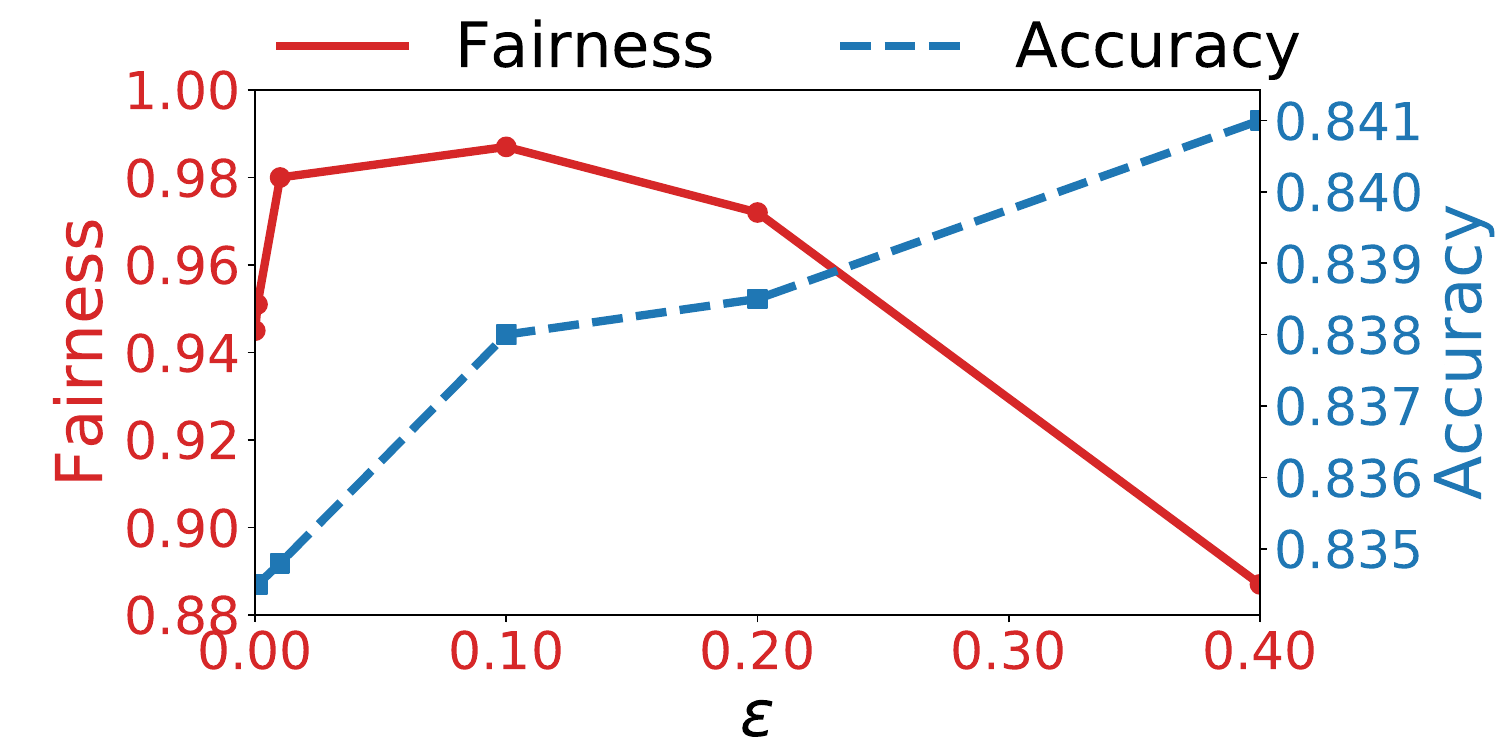}}
\hspace{\figurehspace}
\subfigure[\scriptsize ADULT (IID), NN]{\includegraphics[width=\figurewidthtwo]{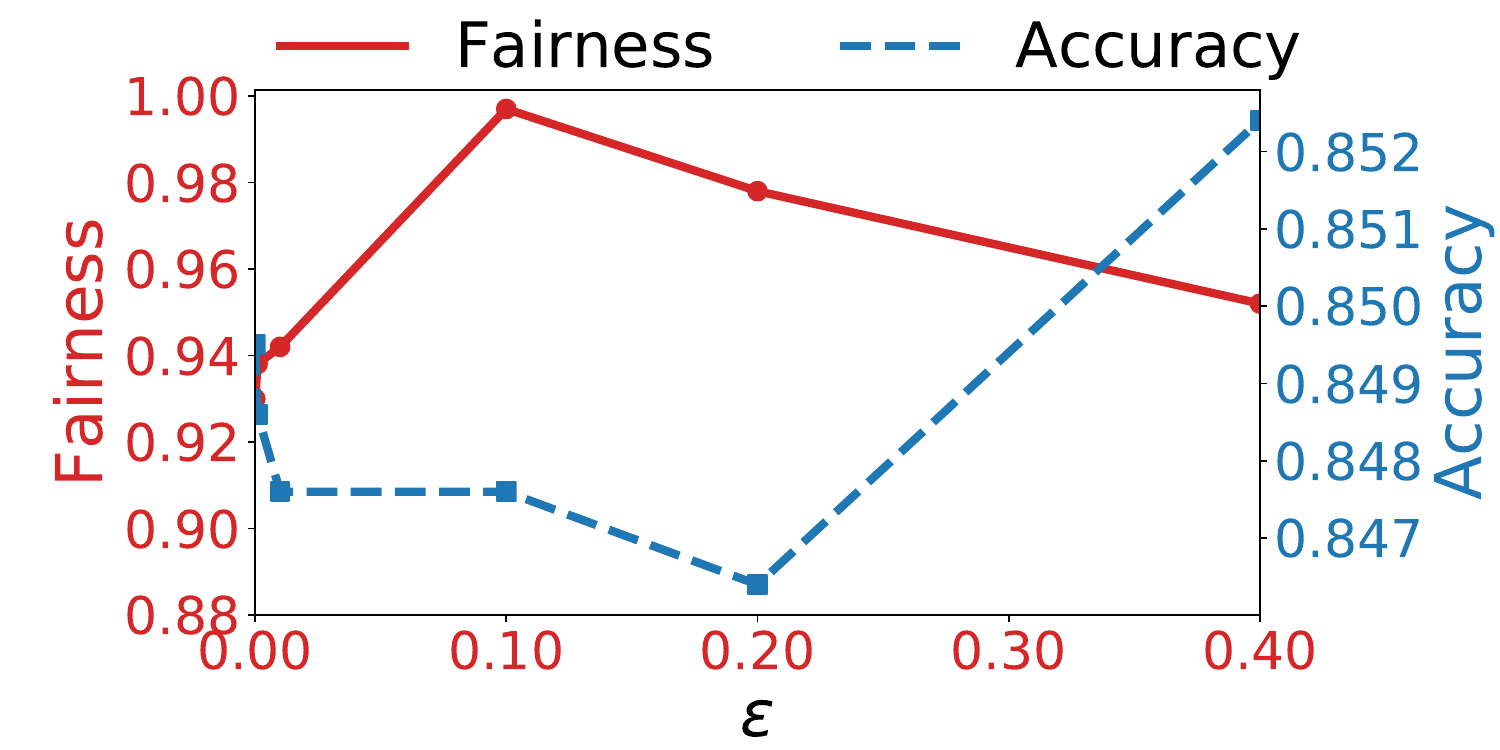}}
\hspace{\figurehspace}
\subfigure[\scriptsize ADULT (non-IID), LR]{\includegraphics[width=\figurewidthtwo]{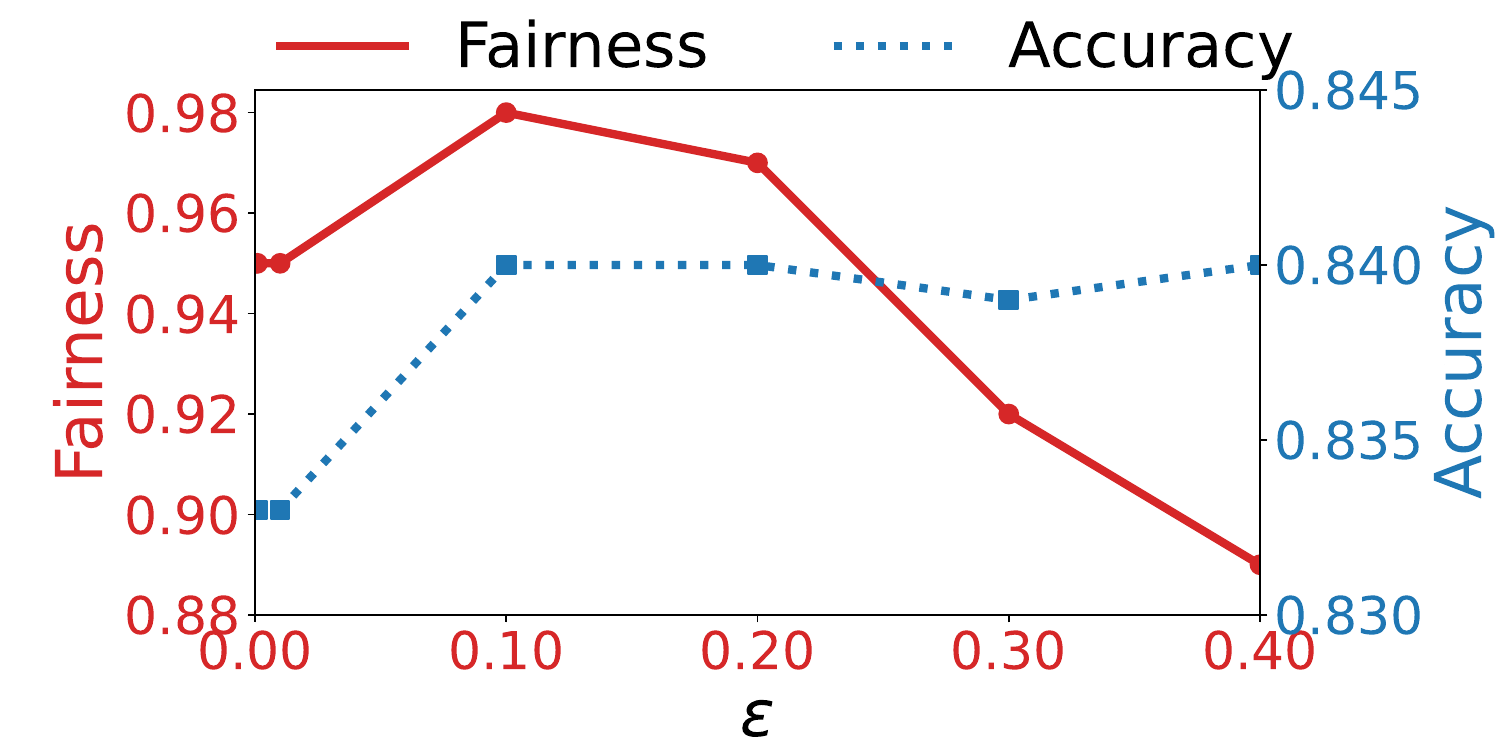}}
\hspace{\figurehspace}
\subfigure[\scriptsize ADULT (non-IID), NN]{\includegraphics[width=\figurewidthtwo]{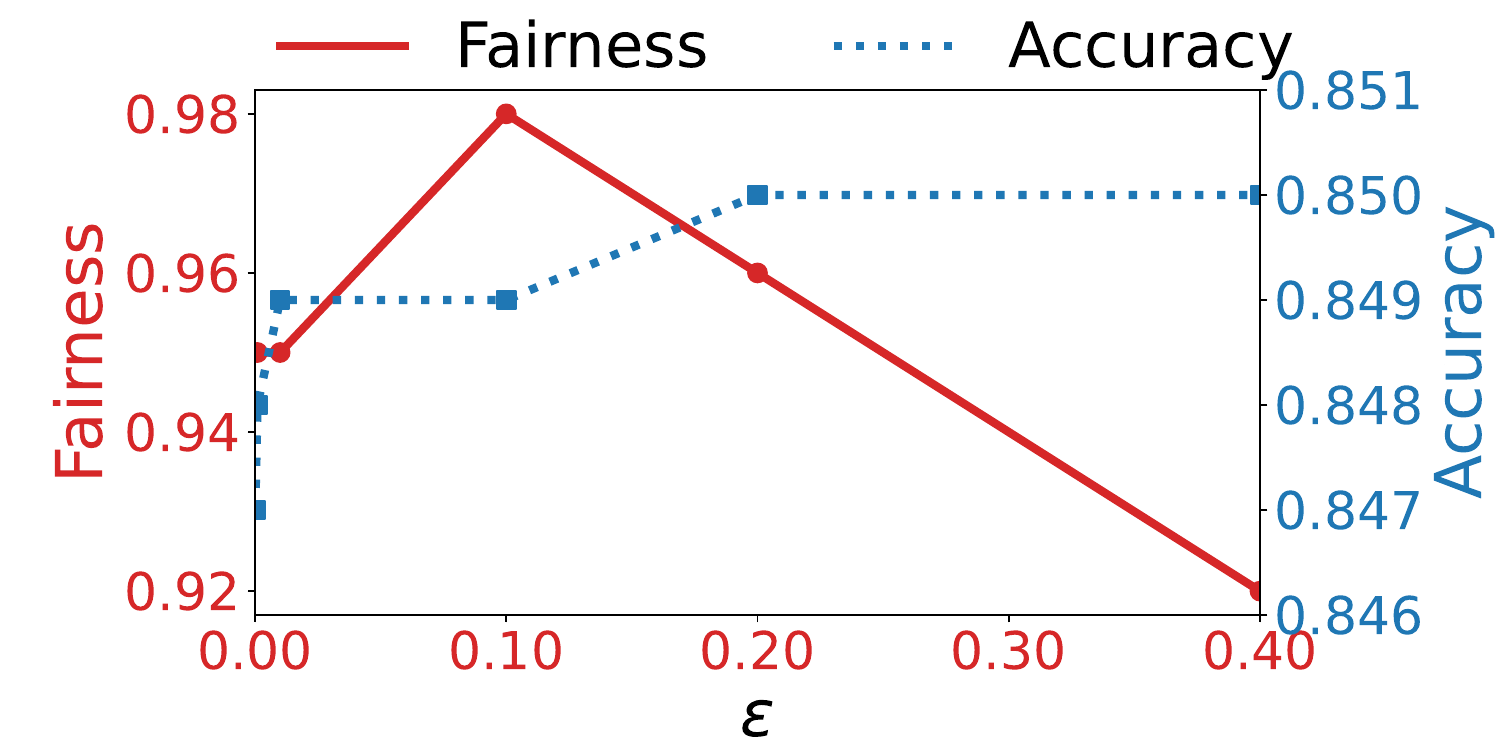}}
% ===================================================
\caption{The effect of parameter $\epsilon$ on fairness and accuracy of the logistic regression (LR) models and neural network (NN) models trained by FedFair.
}
\label{fig:para_anal_acu_fair}
\end{figure*}

As shown in Figure~\ref{fig:overall_perf}, 
the models trained by ST do not perform well in either fairness or accuracy due to the lack of collaboration among clients during training. 
The model trained by FedAvg cannot achieve a good fairness due to the ignorance of fairness during training.
The models post-processed by EO and CEO cannot achieve a good fairness and a good accuracy at the same time, because the improvements on their fairness are usually gained at the cost of sacrificing accuracy~\cite{pessach2020algorithmic}.

The in-processing methods, LCO, AF, and FedFair, 
achieve better results, because they can explicitly impose the required trade-off between accuracy and fairness in the objective functions to find a good balance during training~\cite{woodworth2017learning}.

The accuracies of the models trained by LCO, AF, and FedFair are mostly comparable, 
but FedFair always achieves the best fairness performance. 
This demonstrates the outstanding performance of FedFair in training fair models with high accuracies.

The accuracies of the logistic regression models in Figure~\ref{fig:overall_perf}(a) are slightly better than the accuracies of the corresponding neural network models in Figure~\ref{fig:overall_perf}(b) because the neural network models tend to overfit the small training data of DRUG.

We can also conclude from Figure~\ref{fig:overall_perf} that the fairness of the models trained by LCO is weaker than that of FedFair 
because the LCO problem quickly becomes infeasible when $\epsilon$ gets small, which makes it impossible to improve the fairness of the models trained by LCO.
To the contrary, the FedFair problem stays feasible for small values of $\epsilon$, thus it is able to train models with a much higher fairness than LCO with only slight overhead in accuracy.

%============================================
% \vspace{2mm}
%============================================

Tables~\ref{Table:best_perf_IID} and~\ref{Table:best_perf_non_IID} summarize the performance of all the compared methods on the data sets in the IID setting and the non-IID setting, respectively. 
To analyze the variance of the performance, for each of the IID setting and the non-IID setting, we repeat the random data set construction process introduced in Section~\ref{sec:data} to construct four additional data sets for each of DRUG, COMPAS, and ADULT. 
Every additional data set is randomly constructed using a unique random seed to generate the $N$ private training data sets and the testing data set.
For each of DRUG, COMPAS, and ADULT, including the first data set we constructed at the beginning, in total we have a collection of five randomly constructed data sets for the IID setting and another collection of five randomly constructed data sets for the non-IID setting.

For FedAvg, AF, LCO, and FedFair, we train five models on the five random data sets in the same collection. Then, we evaluate the mean and standard deviation of the accuracies, fairness, and harmonic means of the five models. 
For EO, CEO, and ST, we train one model per client on each of the five random data sets in the same collection. This produces $5N$ models for each collection of random data sets, where $N$ is the number of clients. Then, we evaluate the mean and standard deviation of the accuracies, fairness, and harmonic means of the $5N$ models. For ST, AF, LCO, and FedFair, we use the optimal hyper-parameter that achieves the best harmonic mean performance.
FedAvg, EO, and CEO are evaluated with the default settings because they do not provide a hyper-parameter to control the trade-off between accuracy and fairness.

As shown in Tables~\ref{Table:best_perf_IID} and~\ref{Table:best_perf_non_IID}, 
in most cases, the accuracies of the models trained by FedFair are comparable with the other methods, but FedFair achieves the best fairness in most of the cases and always achieves the best harmonic mean.
This demonstrates the outstanding performance of FedFair in training fair models with high accuracies.

\subsection{Effect of Parameter $\epsilon$}
\label{sec:para_anal_afe_dgeo}

In this subsection, we first analyze how the parameter $\epsilon$ affects the fairness and accuracy of the logistic regression models trained by FedFair, then we analyze the effect of $\epsilon$ on the neural network models trained by FedFair.
%Figure~\ref{fig:para_anal_acu_fair} shows the performance of the logistic regression models and neural network models trained by FedFair. 

Figure~\ref{fig:para_anal_acu_fair} shows the performance of the logistic regression models trained by FedFair. 
For each data set, we use FedFair to train a logistic regression model $f_\theta$ with different values of $\epsilon$. The sets of $\epsilon$ we use are $\{10^{-4}, 10^{-3}, 0.01, 0.07, 0.13, 0.19\}$ for DRUG and $\{10^{-4}, 10^{-3}, 0.01, 0.1, 0.2, 0.4\}$ for COMPASS and ADULT.
We report the performance of $f_\theta$ in fairness and accuracy on the testing data.

As it is shown in Figures~\ref{fig:para_anal_acu_fair}(a), \ref{fig:para_anal_acu_fair}(c), \ref{fig:para_anal_acu_fair}(e), and \ref{fig:para_anal_acu_fair}(g), 
reducing $\epsilon$ significantly improves the fairness of $f_\theta$, since a smaller $\epsilon$ requires $f_\theta$ to have a smaller DGEO, that is, tending to be fairer.
However, a smaller $\epsilon$ does not always produce a better fairness here.
The reason is that the evaluation metric of fairness is defined based on the notion of DEO$(f_\theta)$, 
which is equivalent to DGEO only if a discrete loss function 
$\ell(f_\theta(\mathbf{x}), y)=\mathbbm{1}_{y*f_\theta(\mathbf{x})\leq 0}$
is used~\cite{donini2018empirical}. 
The discrete loss function cannot be used to train the logistic regression model $f_\theta$, so we use a continuous loss function, which makes DGEO not equivalent to DEO$(f_\theta)$.
%Therefore, a smaller $\epsilon$ reduces DGEO, but may not reduce DEO$(f_\theta)$ in some cases.

%As a result, when $\epsilon$ gets smaller than 0.10, the DGEO in Figure~\ref{fig:para_anal_afe_dgeo}(i) decreases, but the corresponding fairness in Figure~\ref{fig:para_anal_acu_fair}(i) drops too.
%A similar phenomenon can also be observed in Figures~\ref{fig:para_anal_afe_dgeo}(k) and~\ref{fig:para_anal_acu_fair}(k).  

The value of $\epsilon$ also has an effect on the accuracy of $f_\theta$.
As it is shown in Figures~\ref{fig:para_anal_acu_fair}(a), \ref{fig:para_anal_acu_fair}(e), and \ref{fig:para_anal_acu_fair}(i), on the data sets in the IID setting. The accuracy of $f_\theta$ drops when $\epsilon$ becomes smaller, because reducing $\epsilon$ induces a tighter federated DGEO constraint, which shrinks the feasible region of the FedFair problem and makes it harder to find a good solution.
As shown in Figure~\ref{fig:para_anal_acu_fair}(c),
%, \ref{fig:para_anal_acu_fair}(d), and \ref{fig:para_anal_acu_fair}(f), 
on the DRUG data set in the non-IID setting, the accuracy of $f_\theta$ does not always drop when $\epsilon$ becomes smaller, because the local data sets on different clients are not identically and independently distributed.
However, this does not affect the practical performance of FedFair much. As discussed in Section~\ref{sec:fmtp}, FedFair always achieves the best performance on all data sets in both the IID and the non-IID settings.

Next, we analyze the effect of $\epsilon$ on the performance of the neural network models trained by FedFair in Figure~\ref{fig:para_anal_acu_fair}. 
The values taken by $\epsilon$ are $\{10^{-4}, 10^{-3}$, 0.005, 0.01, 0.02, 0.07, 0.13, 0.19, 0.22, 0.25$\}$ for DRUG and $\{10^{-4}, 10^{-3}, 0.01, 0.1, 0.2, 0.4\}$ for COMPAS and ADULT.

We can see in Figure~\ref{fig:para_anal_acu_fair} that the accuracies of the neural network models do not always drop with the decrease of $\epsilon$. This is largely due to the high model complexity and high non-linearity of the neural network models, where a tighter federated DGEO constraint induced by a smaller $\epsilon$ may act as a regularization term to improve the testing accuracies of the neural network models.
We can also observe that a smaller $\epsilon$ does not always induce a larger fairness for the neural network models.
This is because the notion of DEO$(f_\theta)$, from which we develop fairness, is not equivalent to DGEO when we use cross-entropy loss to train neural network models~\cite{donini2018empirical}.

%% file: Chapters/con.tex
\section{Conclusions}
\label{sec:con}
In this paper, we tackle the task of training fair models in federated learning. We first propose an effective federated estimation method to accurately estimate the fairness of a model and analyze why it achieves a smaller estimation variance than locally estimating the fairness of a model on every client.
Based on the federated estimation, we develop a fairness constraint that can be smoothly incorporated into an effective federated learning framework to train high-performance fair models without infringing the data privacy of any client.